\documentclass[runningheads]{llncs}
\usepackage[T1]{fontenc}
\usepackage{graphicx}
\usepackage{color}
\usepackage{multirow}
\usepackage{bm}
\usepackage[ruled]{algorithm2e}
\usepackage{graphicx}
\graphicspath{{}}
\usepackage{amsmath}

\begin{document}
\title{Secure Split Learning against Property Inference, Data Reconstruction, and Feature Space Hijacking Attacks}
\titlerunning{Secure Split Learning}
%
\author{Yunlong Mao\inst{1} \and
Zexi Xin\inst{1} \and
Zhenyu Li\inst{2} \and
Jue Hong \and
Qingyou Yang \and \\
Sheng Zhong\inst{1}
}
\authorrunning{Y. Mao, Z. Xin, et al.}
%
\institute{State Key Laboratory for Novel Software Technology, Nanjing University, China \and
University of California San Diego, USA}
\maketitle              
\begin{abstract}
Split learning of deep neural networks (SplitNN) has provided a promising solution to learning jointly for the mutual interest of a guest and a host, which may come from different backgrounds, holding features partitioned vertically. However, SplitNN creates a new attack surface for the adversarial participant, holding back its practical use in the real world. By investigating the adversarial effects of highly threatening attacks, including property inference, data reconstruction, and feature hijacking attacks, we identify the underlying vulnerability of SplitNN and propose a countermeasure. To prevent potential threats and ensure the learning guarantees of SplitNN, we design a privacy-preserving tunnel for information exchange between the guest and the host. The intuition is to perturb the propagation of knowledge in each direction with a controllable unified solution. To this end, we propose a new activation function named R\textsuperscript{3}eLU, transferring private smashed data and partial loss into randomized responses in forward and backward propagations, respectively. We give the first attempt to secure split learning against three threatening attacks and present a fine-grained privacy budget allocation scheme. The analysis proves that our privacy-preserving SplitNN solution provides a tight privacy budget, while the experimental results show that our solution performs better than existing solutions in most cases and achieves a good tradeoff between defense and model usability.


\end{abstract}
\section{Introduction}
Private data, such as biological information and shopping history, is potentially valuable for commercial usage. Commercial agencies can learn users' preferences and make proper recommendations using their private data. Meanwhile, users could enjoy personalized services by sharing their privacy with service providers. However, both agencies and users are worried about their private data being abused. Besides, since commercial agencies nowadays are dedicated to providing high-quality services in vertical industries, like social networks or online shopping, their user profiles are business-relevant and highly homogeneous. But it commonly requires diversified features for building deep learning models. For these reasons, using multifarious private data for building a satisfying model is essential but challenging.

Fortunately, the collaborative learning paradigm \cite{mcmahan2017communication} has emerged as a promising solution. Collaborative learning enables participants from different interests to learn a shared model jointly. A well-known collaborative learning paradigm is federated learning \cite{mcmahan2017communication}, focusing on the coordination of distributed participants. Meanwhile, another paradigm split neural network (SplitNN for short) \cite{gao2020end,ceballos2020splitnn} is designed explicitly for vertically partitioned features. 
Participants from different backgrounds could contribute with distinct feature representations. By combining different features, SplitNN is supposed to be more expressive in the real world. Notably, SplitNN has already been used for building industry-level frameworks, such as FATE \cite{fate} and Syft \cite{syft}.

However, collaborative learning paradigms are faced with severe security issues. Roughly speaking, there are three kinds of threats, inference attack \cite{nasr2019comprehensive,salem2019ml}, reconstruction attack \cite{hitaj2017deep,salem2020updates}, and poisoning attack \cite{tolpegin2020data,HuangMGL0X21}. An inference attack discloses attributes or membership information of specific data samples of the participants, while a reconstruction attack seeks to generate data samples similar to participants' private data. In \cite{erdogan2021unsplit}, the reconstruction attack and the inference attack are studied simultaneously, enabling the host to recover the client's data and allowing an honest-but-curious host to infer the labels with desirable accuracy. Unlike these two kinds of threats, a poisoning attack aims to put harmful data into collaborative learning for malicious purposes rather than stealing private information. Moreover, a feature space hijacking attack considers a malicious participant, enlarging the attack effect of inference and reconstruction. Some excellent defensive solutions have been proposed for federated learning since it is a more general paradigm. According to the techniques used, these solutions can be roughly classified into differential privacy solutions \cite{sunppaifl21,liu2021flame} and secure multiparty computation solutions \cite{zheng2022aggregation,nguyen2021flguard}.

Unfortunately, security issues in SplitNN are barely discussed. As we illustrate in Figure~\ref{fig:splitnn}, the workflow of SplitNN has a unique asymmetric design. Therefore, most solutions for secure federated learning are not suitable for SplitNN. Secure multiparty computing solutions like homomorphic encryption can achieve ideal data confidentiality \cite{zhang2020batchcrypt,pereteanu2022split}, but the overhead introduced is still far away from practical uses. It has been proved in \cite{gawron2022feature} that the defense performance of applying differential privacy in a general way to split learning is far away from the expectation when dealing with a feature space hijacking attack (FSHA) \cite{Pasquini2021UnleashingTT}. To defend against FSHA, two novel methods are designed in \cite{erdogan2023defense} for a split learning client to detect if it is being targeted by a hijacking attack or not. One approach is an active method relying on observations about the learning object, and the other one is a passive method at a higher computing cost. Thus, for the first attempt at privacy-preserving SplitNN for both the server and the client, we offer a unified solution for addressing different threats, including inference, reconstruction, and hijacking attacks. We focus on the abovementioned threats since they share a similar adversarial goal of privacy disclosure, while poisoning attacks need to be studied separately \cite{fang2020influence,mao2021romoa}.
\begin{figure}
    \centering
    \includegraphics[scale=0.6]{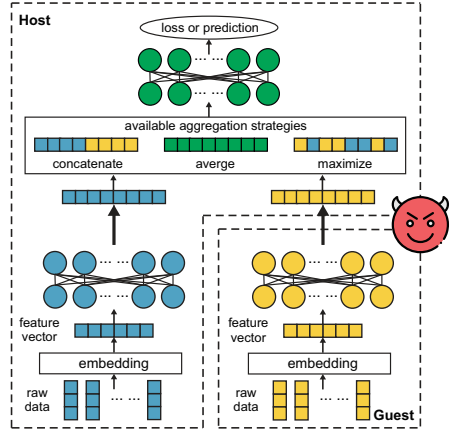}
    \caption{An illustration of SplitNN architecture.}
    \label{fig:splitnn}
\end{figure}

We also note that there is an inherent contradiction \cite{erdogan2021unsplit} between privacy preservation and model usability, especially when private information of two sides should be considered in SplitNN. Through the investigation of the attack effect, we find that the attacker has sufficiently high success rates when disclosing the privacy of either party. To defeat the attacks \cite{hitaj2017deep,salem2019ml,Pasquini2021UnleashingTT} and give privacy guarantees on both sides in SplitNN, we make the following contributions:
\begin{itemize}
    \item We investigate the privacy leakage issue in SplitNN by adapting inference and reconstruction attacks from federated learning. For the first attempt at securing SplitNN against multiple attacks, we propose a unified solution based on a newly designed activation function.  
    \item We offer strong privacy guarantees for both sides of SplitNN. Moreover, a fine-grained privacy budget allocation scheme is given to achieve more efficient perturbations and improve privacy budget utilization.
    \item We implement and evaluate our solution using real-world datasets for different split learning tasks. The experimental result shows that our solution outperforms existing solutions when concurrently considering privacy preservation and model usability.
\end{itemize}


\section{Problem Statement}
\subsection{Deep Learning}
Given a training dataset $\bm{X}$ and DNN model parameters $\bm{\theta}$, a training task is to find approximately optimal $\bm{\theta}$ by minimizing a pre-defined loss function $\mathcal{L}$, regarding input $\bm{X}$. We assume that the optimizer used is a mini-batch stochastic gradient descent (SGD) algorithm, which updates $\bm{\theta}$ with a batch input of $\bm{X}$ iteratively. Assuming the batch size is $M$, then the total loss of $\bm{\theta}$ for a batch input $\bm{x} = \{ x_i | x_i \in \bm{X}, i \in [1,M] \}$ should be $\sum_{x \in \bm{x}} \mathcal{L}(\bm{\theta},x)$ in the $t$-th training iteration. The gradients of $\bm{\theta}$ for model updating should be estimated by $\frac{1}{M} \sum_{x \in \bm{x}} \nabla_{\bm{\theta}} \mathcal{L} (\bm{\theta}, x)$ approximately. Hence, parameters $\bm{\theta}$ can be updated as $\bm{\theta}^{t+1} = \bm{\theta}^{t} - \frac{1}{M} \sum_{x \in \bm{x}} \nabla_{\bm{\theta}} \mathcal{L} (\bm{\theta}, x)$. This mini-batch SGD-based optimizing procedure should be repeated until the model usability meets the requirement or the maximal count of iterations reaches.

\subsection{Split Learning}
SplitNN is an emerging collaborative learning paradigm partitioning the original neural network into different parts. There are commonly two types of participants in SplitNN, the host and the guest. For each training iteration, the forwarding input of the guest should be evaluated locally and passed to the host. Then the backpropagation should be initialized by the host and propagated to the guest. According to related studies \cite{gupta2018distributed}, there exist several configurations of SplitNN. This paper will focus on the SplitNN design suitable for building models jointly with vertically partitioned features \cite{ceballos2020splitnn}. An example of SplitNN for such collaboration is shown in Figure~\ref{fig:splitnn}.

Assume that host and guest in Figure~\ref{fig:splitnn} are two companies aiming to predict customer behaviors collaboratively. After an embedding procedure, the host who holds label data merges embedded vectors using a predefined strategy, say averaging. Then the host will finish the rest of the forward propagation and initiate backward propagation. In Table~\ref{tab:benchmark}, we give a benchmark of SplitNN using different merging strategies for building recommendation models with two public datasets, MovieLens \cite{harper2015Movielens} and BookCrossing \cite{ziegler2005improving}. For misaligned features where the server and the client may hold different shapes of features for the same data entry, we use zero padding to complement the missing features by zeros to retain the same shape of feature vectors. We give the top-10 hit ratio for the test in Table~\ref{tab:benchmark}, which are the average of 30 runs in the same setting. During the experiment, we divide the original feature vector of 160 dimensions into two parts. One is 96 dimensional while the other is 64 dimensional. We notice that merging strategies have no significant influence on performance. Thus, averaging will be used as the default setting.
\begin{table}[ht]
\centering
\caption{Top-10 hit ratio (\%) of SplitNN using different merging strategies. Batch size 32, learning rate 0.01. Model architectures are shown in the appendix.}
\label{tab:benchmark}
\begin{tabular}{ccccccccc}
\hline
\multicolumn{2}{c}{\multirow{2}{*}{}}                             & \multirow{2}{*}{concat} & \multicolumn{5}{c}{element-wise}                 & \multirow{2}{*}{no split} \\ \cline{4-8}
\multicolumn{2}{c}{}                                              &                              & max     & sum     & avg. & mul. & min     &                           \\ \hline
\multicolumn{1}{c|}{\multirow{2}{*}{MovieLens}}                     & padding     & 56.62 & 56.26 & 56.35 & 56.89 & 55.23 & \textbf{57.19} & \multirow{2}{*}{57.21}  \\ \cline{2-8}
\multicolumn{1}{c|}{}   & non-padding & 55.38 & 54.75 & 54.95 & \textbf{55.72} & 54.52  & 55.08 &    \\ \hline
\multicolumn{1}{c|}{\multirow{2}{*}{Book Crossing}} & padding  & \textbf{61.70} & 60.84 & 60.21 & 61.16 & 58.99 & 60.98 & \multirow{2}{*}{61.92}  \\ \cline{2-8}
\multicolumn{1}{c|}{}  & non-padding & 58.80 & 59.34 & \textbf{59.44} & 59.10 & 58.85 & 59.02 &      \\ \hline
\end{tabular}
\end{table}

\subsection{Threat Model}
Unlike updating local models separately in federated learning, participants of SplitNN are required to update local models cooperatively. Interactions between two parties pose new threats to each other \cite{fu2022label,li2022ressfl}. Hence, we will investigate privacy leakage threats from two perspectives. The host and the guest are honest but curious about the private data of each other. They are allowed to do any additional computations when they are following the split learning protocol. Assuming both parties are rational and privacy-aware, they will not exchange information except for the interactive interface shown in Figure~\ref{fig:splitnn}. Please note that label leakage attacks and defense in SplitNN have attracted much attention recently. However, these topics are out of our discussion and need to be studied separately. Therefore, both parties may carry out certain attacks to infer or reconstruct each other's private data even simultaneously. The following part introduces the attacks which are considered potential threats against one party or both parties.

\textbf{\textit{Property inference attack}}. Since the host has access to the output of the guest while the guest receives gradients containing private information of the host, an adversary can mount the property inference attack \cite{ganju2018property,luo2021feature} from both parties. Access to the output of the local model on the other side can be seen as a black-box query. In this setting, the adversary can infer properties of private data by observing query input and the corresponding output. By constructing elaborating shadow models, the adversary can steal substantial information from the target. In this way, the adversary acquires the capability of inferring some properties (such as gender and age) of the data samples used for training. Denoted by $F$, $T$, $\mathcal{L}_{F}$ and $l_i$ the inference model, target model, the loss function used for $F$ and the label of each data, the adversarial goal is
\begin{equation}
    \mathcal{A}_{PIA} = \arg\min_{F} \sum_{x_i \in \bm{X}} \mathcal{L}_{F} (F(T(x_i)), l_i), l_i \in \{0,1\}.
\end{equation}

\textbf{\textit{Data reconstruction attack}}. A generative adversarial network (GAN) is an instance of generative models designed to estimate target data distribution \cite{goodfellow2014generative}. Taking advantage of GANs, a data reconstruction attack is proposed in \cite{hitaj2017deep}. The adversary of a reconstruction attack aims to reconstruct the private training data of other participants in collaborative learning. The host or the guest or both can be adversarial. To mount the attack, the adversary augments the training data per iteration by inserting fake samples $Z$ generated by a generator $G$. The global model will serve as a discriminator $D$. The adversary will affect global model updating by deceiving the target using fake training samples. For correcting the adversary, the target participant is supposed to put more private information into the learning. In this game-style training, the adversary may obtain substantial knowledge to reconstruct data samples as similar to target data as possible. Thus, the adversarial goal can be given as
\begin{equation}
\mathcal{A}_{DRA} = \min_{G} \max_{D} \frac{1}{|X|} \sum_{x \in X} log D(x) + \frac{1}{|X|} \sum_{z \in Z} log(1-D(G(z))).
\end{equation}

\textbf{\textit{Feature space hijacking attack}}. Unlike attacks mentioned above, the feature space hijacking attack \cite{Pasquini2021UnleashingTT} considers a malicious adversary capable of manipulating the learning process. With the help of hijacking, the adversary can improve the performance of inference and reconstruction attacks. In this setting, the adversary can only be the host because label information is needed to mislead the victim. To mount the attack, the adversary uses a pilot model $\hat{f}$, an approximation of its inverse function $\hat{f}^{-1}$ with a shadow dataset and a discriminator $D$ to distinguish the output from guest model $f$ and the pilot model $\hat{f}$ during the training process. Then, the malicious host can send a suitable gradient to hijack the training of an honest guest by setting the goal as
\begin{equation*}
\mathcal{A}_{FSHA} = log(1-D(f(X_{PRIV}))).
\end{equation*}

%

\section{Privacy-Preserving Split Learning}
Our solution will be designed to preserve the privacy of the host and the guest concurrently. Ideally, the guest wants to collaborate with the host under the condition that the host should disclose no private data and vice versa. However, unlike conventional model publishing scenarios, the host and guest in SplitNN, like in Figure~\ref{fig:splitnn}, are required to exchange intermediate results continually. These continuous queries significantly increase the risk of privacy leakage for both sides. Moreover, the attack surface of SplitNN is inside the neural network, which is different from situations studied in end-to-end models \cite{mao2020private,yu2019differentially}.

Our work offers the first privacy-preserving SplitNN solution for defending against multiple attacks from two directions. The fundamental idea is to construct a bidirectional privacy-aware interface between the host and the guest. Noting that components of a neural network are loosely coupled, any output port may be a candidate for the interface. However, recent studies have proved that activation functions are more adaptive for perturbed operands \cite{gao2020adaptive}. Moreover, activation functions have various forms, which are flexible for configuration. As a result, we design a new variant of ReLU as an interface for SplitNN.

\subsection{R\textsuperscript{3}eLU: Randomized-Response ReLU}
Inspired by randomized response mechanisms \cite{warner1965randomized}, we design a new activation function R\textsuperscript{3}eLU (randomized-response ReLU) for SplitNN. Specifically, R\textsuperscript{3}eLU consists of a randomized-response procedure \cite{warner1965randomized} and a Laplace mechanism \cite{dwork2014algorithmic}. This combination is not arbitrary but a complementary result. The original randomized response is good at statistical analysis of item sets. But the activation result of an input sample is commonly a continuous variable. The Laplace mechanism is a classic approach for differential privacy, handling continuous variables. But the perturbation is hard to be controlled, especially when the sensitivity degree of a query function is relatively large. It also means that it is highly risky to adopt the Laplace mechanism to an activation function directly. In SplitNN, we consider the model held by each party as a query function. Remember that both parties need to protect their privacy, so the sensitivity is bounded by the output of the cut layer in the process of forward propagation for the guest and by the gradient in the process of backward propagation for the host.

Recall that the original ReLU is $f(v) = \max(0,v)$, $v\in\mathcal{R}$. A randomized-response variant should yield a proper substitute for replacing real activations with a probability of $p$. We consider the activations of the cut layer as item sets and apply randomized-response on them. If we yield $0$ as the substitute for $v > 0$, then we can inactivate a part of ReLU results, serving as artificial perturbations. But nothing has been changed for $v \leq 0$. Thus, it is not privacy-preserving since $f(v)=0$ also reveals private information, indicating $v \leq 0$. To enforce a strict privacy policy, we integrate a Laplace mechanism into the ReLU variant by adding noise $z \xleftarrow{r} Laplace(0,\sigma)$. In this way, we can give the definition of R\textsuperscript{3}eLU as
\begin{equation}
\text{R}^3\text{eLU}(v) = \left\{
    \begin{split}
        & \max{(0,v + z)}, \ && \text{with probability} \ p, \\
        & 0, \ && \text{with probability} \ (1-p).
    \end{split}
    \right.
\end{equation}


\subsection{Forward Propagation with R\textsuperscript{3}eLU}
In forward propagation, the guest needs to transfer local forwarding results to the host. At this point, an adversarial host can mount property inference or reconstruction attacks. To stem the leakage, we recommend replacing the original activating function with R\textsuperscript{3}eLU while leaving the rest unchanged. Algorithm~\ref{alg:re3lu} gives R\textsuperscript{3}eLU-forward procedure by integrating essential operations. Denoted by $\bm{v}^{g}$ (the superscript may be omitted for concision) input of the original ReLU of the guest and $N$ the cardinality of $\bm{v}$. The first operation is to select the top $K$ largest elements of $\bm{v}$ and zero the rest. Then the top-K elements are clipped by a hyper-parameter $C$. The abovementioned pre-processing is defined as a procedure $ClipK$, taking as input $\bm{v}$, constants $K$ and $C$, outputting $\hat{\bm{v}}$. We pre-process the inputs of $\text{R}^3$eLU to bound the sensitivity. The method $ClipK$ preserves the maximum $K$ absolute values and clips each vector in the $l_1$ norm for a clipping threshold $C$. For randomized responding activation states, we calculate the probability
\begin{equation}
p_i = \frac{1}{2}+\frac{\hat{v}_i}{\Vert \bm{\hat{v}} \Vert_\infty} \cdot (\frac{e^{\frac{\epsilon_p}{K}}}{1 + e^{\frac{\epsilon_p}{K}}} - \frac{1}{2}),
\end{equation}
where $\epsilon_p$ indicates the privacy budget of randomized responding. The state of $\hat{v}_i$ will be deactivated with probability $1 - p_i$ as per R\textsuperscript{3}eLU definition. Finally, a Laplace mechanism with privacy budget $\epsilon_l$ is integrated into the R\textsuperscript{3}eLU-forward completing the procedure. Now, the guest will transmit $\tilde{\bm{a}}^g = \text{R}^3\text{eLU-forward}(\bm{v}^g$, $C, K, N$, $\epsilon_p, \epsilon_l)$ instead of $\bm{a}^{g} = \text{ReLU}(\bm{v}^g)$ to the host.

\begin{algorithm}[ht]
\caption{R\textsuperscript{3}eLU-forward procedure.}
\label{alg:re3lu}
\LinesNumbered
\DontPrintSemicolon
\KwIn{original input $\bm{v}^g$, cardinality $N$, constants $C$ and $K$, privacy parameters $\epsilon_p$ and $\epsilon_l$, probability $\bm{p}$.}
\KwOut{activation $\bm{\tilde{a}}^g$.}
\emph{$\bm{\hat{v}}^g \leftarrow ClipK(\bm{v}^g, C, K, N)$ \tcp*[r]{pre-process}}
\For{$i \leftarrow 1$ \KwTo $N$}{
    $r \xleftarrow{r} \mathcal{N}(0,1)$ \\
    \eIf{$r < p_i$}{
        $\tilde{a}^g_i \leftarrow \max(\hat{v}^g_i + \texttt{Lap}(0, \frac{2KC}{\epsilon_l}), 0)$ \tcp*[r]{activate}
    }{
        $\tilde{a}^g_i \leftarrow 0$ \tcp*[r]{deactivate}
    }
}
\Return $\tilde{\bm{a}}^g \leftarrow \{ \tilde{a}^1, \tilde{a}^2, \ldots, \tilde{a}^N \}$
\end{algorithm}

\subsection{Private Backward Propagation}
Privacy leakage also exists from the host's perspective. When the host finishes the rest of forwarding propagation after aggregating activations $\tilde{\bm{a}}^{g}$ and $\bm{a}^{h}$, the loss produced for backward propagation contains data privacy of the host and guest. According to recent studies of backward propagation \cite{melis2019exploiting,salem2020updates}, intermediate results of model updating can cause severe data privacy leakage. Since the loss must be propagated to the guest, it is crucial to prevent the host from being attacked by an adversarial guest. However, the partial loss propagated ranges widely. Integrating a DP mechanism directly into the original ReLU is unrealistic. Due to the randomized-response design, we can construct a privacy-preserving tunnel for backward propagation atop the derivative of R\textsuperscript{3}eLU.

Recall that the derivative value of ReLU for any input is either one or zero. A randomized-response variant will perturb the binary output randomly. Besides, randomly flipping still discloses real partial losses when value ones are not flipped. Therefore, a Laplace mechanism is used in the backward procedure. Now, we give a randomized-response derivative of R\textsuperscript{3}eLU
\begin{equation}
\nabla \text{R}^3\text{eLU}(\bm{\delta}^{g}, \tilde{\bm{a}}^{g}, \bm{v}^{g}) = \left\{
    \begin{split}
        & \bm{\delta}^{g} + \bm{z}, \ && \text{with probability} \ p, \\
        & 0, \ && \text{with probability} \ (1 - p),
    \end{split}
    \right.
\end{equation}
where $\bm{\delta}^{g}$ is the partial loss for the guest model, $\bm{z}$ is artificial noise. Similar to R\textsuperscript{3}eLU-forward, R\textsuperscript{3}eLU-backward also needs some essential operations. Thus, the same $ClipK$ process for top-K selecting and scalar clipping can be adopted for R\textsuperscript{3}eLU-backward. However, different from R\textsuperscript{3}eLU, absolute values are used because the partial loss instructs the gradient descent direction. In this case, we have $\hat{\bm{\delta}}^{g}=ClipK(\bm{\delta}^{g}, C, K, N)$. Moreover, a $Sign$ process is used to obtain signs. 

For randomized responding, the probability of retaining an actual loss is
\begin{equation}
    p_i = \frac{1}{2}+\frac{|\hat{\delta_i}|}{\Vert \hat{|\delta_i|} \Vert_\infty} \cdot (\frac{e^{\frac{\epsilon_p}{K}}}{1 + e^{\frac{\epsilon_p}{K}}} - \frac{1}{2}),
\end{equation}
where $\epsilon_p$ is the privacy budget for the randomized response. We now give the backward procedure R\textsuperscript{3}eLU-backward in Algorithm~\ref{alg:re3lu-diff}. Please note that although some existing solutions choose to disturb the gradients of two parties, we choose to perturb the partial loss regarding the guest's backpropagation while keeping the partial loss of the host model unchanged. In this way, a slighter influence is caused for the host compared with disturbing all gradients directly.
\begin{algorithm}[ht]
\caption{R$^3$elu-backward procedure.}
\label{alg:re3lu-diff}
\LinesNumbered
\DontPrintSemicolon
\KwIn{partial loss $\bm{\delta}^g$, cardinality $N$, constants $C$ and $K$, privacy parameters $\epsilon_p$ and $\epsilon_l$, probability $\bm{p}$.}
\KwOut{partial loss $\tilde{\bm{\delta}}^g$.}
\emph{$\hat{|\bm{\delta}|} \leftarrow ClipK(|\bm{\delta}|, C, K, N)$} \tcp*[r]{pre-process}
\For{$j \leftarrow 1$ \KwTo $N$}{
    $r \xleftarrow{r} \mathcal{N}(0,1)$ \\
    \eIf{$r < p_i$}{
        $\hat{\delta}_i \leftarrow Sign(\delta_i) \cdot |\hat{\delta_i}| $
    }{
        \emph{$\hat{\delta}_i \leftarrow 0 $ \tcp*[r]{randomized response}}
    }
    $\tilde{\delta}^g_i \leftarrow \hat{\delta}^g_i + \texttt{Lap}(0, \frac{2KC}{\epsilon_l})$
    }

\Return $\tilde{\bm{\delta}}^g = \{ \tilde{\delta}^1, \tilde{\delta}^2, ..., \tilde{\delta}^N \}$
\end{algorithm}

\subsection{Dynamic Privacy Budget Allocation}
To further reduce privacy loss and improve the utilization of the privacy budget, we recommend allocating the privacy budget for parameters dynamically instead of allocating uniformly. Based on \cite{molchanov2019importance}, the importance of a parameter during training can be quantified by the error introduced when it is removed from the model. In particular, the importance $I_j$ of $\theta_j \in \bm{\theta}$ is the squared difference of prediction errors caused by removing $\theta_j$, i.e.,
\begin{equation}
    I_{j}=(\mathcal{L}(\bm{x},\bm{\theta})-\mathcal{L}(\bm{x},\bm{\theta} \setminus \{ \theta_j \}))^2.
\end{equation}
For efficiency concern, an approximating method is given in \cite{molchanov2019importance}, estimating the importance $I_j$ by its first-order Taylor expansion as
\begin{equation}
    \hat{I}_{j}=(\nabla_{\theta_j} \mathcal{L} (\bm{\theta}, x) \cdot \theta_j)^2.
\end{equation}

Given the importance of each parameter in the cut layer, the importance of a feature can be derived further. Specifically, the importance of a feature $U_{j}$, $j \in [1, N_u]$, where $N_u$ is the total number of neurons in the cut layer, can be calculated as joint importance of relevant parameters by summing them up. Thus, $U_{j} = \sum\nolimits_{\theta_k \in \bm{\theta}_{U_{j}}} \hat{I}_{k}$, where $\bm{\theta}_{U_{j}}$ is the set of all parameters directly connected to the $j$-th neuron.

Please note that the original importance estimation is designed for a well-trained model and cannot be directly applied to intermediate models during training. To tackle the problem, we give a dynamic estimation method by deriving the original method into a cumulative form. The importance of a feature will be accumulated as the training epoch increases. Specifically, the importance of the $j$-th neuron in the $q$-th training epoch is
\begin{equation}
\small
    U_{j}^{q} = \frac{\sum\nolimits_{\theta_{k} \in \tilde{\bm{\theta}}_{j}} \hat{I}_{k} + U_{j}^{q-1} \times (q \times \lfloor T/n_t \rfloor + (t \bmod n_t) - 1)}{q \times \lfloor T/n_t \rfloor + (t \bmod n_t)},
\label{equ:9}
\end{equation}
where $n_t$ indicates the iteration number within a training epoch, $T$ is the maximum training iteration number, and $t$ is the current training iteration. Assuming that $T \mod n_t = 0$, then $q \in [1,T/n_t]$.

Based on the importance estimated, now we can dynamically allocate the privacy budget for different features. The intuition is to give larger budgets to more important features. Before the $q$-th training epoch begins, we can estimate a feature importance vector $\bm{U} = \{ U_{1}^{q}, U_{2}^{q}, \ldots, U_{N_u}^{q} \}$. Accordingly, the privacy budget allocated to each feature will be $\epsilon_{j} = \epsilon \times U_{j}^{q}$, if $\epsilon$ is one unit budget. Then the total privacy budget for all features is $\epsilon_{F} = \sum_{j = \in [1,N_u]} \epsilon_{j}$. Now, we can set the probability of randomized response for R\textsuperscript{3}eLU-forward and R\textsuperscript{3}eLU-backward using the dynamic budget allocation, 
\begin{equation}
p_i = \frac{1}{2}+\frac{U_i^q}{\Vert \bm{U} \Vert_\infty} \cdot (\frac{e^{\frac{\epsilon_p}{K}}}{1 + e^{\frac{\epsilon_p}{K}}} - \frac{1}{2}).
\end{equation}

On the other hand, we can also allocate different privacy budgets for different iterations for better budget utilization. Given the total budget $\epsilon_T$ for all iterations, we assign the privacy budget $\epsilon_i = \frac{\epsilon_T}{2^i}$ to the $i$-th iteration as suggested by \cite{du2021dynamic}. Since $\sum_{i=1}^{\infty} \frac{\epsilon_T}{2^i} = \epsilon_T$, according to the sequential composition theory of differential privacy, we can ensure that the whole training process achieves $\epsilon_T$-differential privacy.

We note that the additional computing cost will be caused by the dynamic privacy budget allocation, which is dominated by the computation of the importance of each neuron at the cut layer. As the gradient of the cut layer can be preserved during the process of each backward propagation, the cost generated by the product of the gradients and neurons for each round is $O(N_u)$ where $N_u$ is the number of neurons of the cut layer.

\section{Privacy Analysis}
We give privacy analysis regarding the host and the guest, respectively. The analysis of privacy loss after dynamic budget allocation will also be given.
\begin{corollary}
In R\textsuperscript{3}eLU-forward procedure, the sensitivity of $\bm{\hat{v}}^g$ is bounded by $2KC$, where $C$ is the clipping constant and $K$ is the number of top values reserved.
\label{cor:step1}
\end{corollary}

\begin{proof}
\normalfont
Given the clipping constant $C$, each element can be bounded by $0 \le \hat{v}^g_i \le C$. Since at most $K$ elements are greater than 0, the sensitivity $\Delta$ of $\bm{\hat{v}}^g$ can be bounded by
\begin{equation}
    \Delta = \max \limits_{i,j \in [1,N]} \{\Vert \hat{v}^g_i - \hat{v}^g_j \Vert_1 \}  \le 2K \vert \max\limits_{i \in [1,N]} \{\hat{v}^g_i\} - \min\limits_{j \in [1,N]} \{\hat{v}^g_j\} \vert  = 2KC
\end{equation}
\end{proof}

In the randomized-response phase, we randomly flip the activating state for input. Briefly, the activating states of $\hat{\bm{v}}^g$ after randomized response are denoted by $\bm{s} = \{s_1, s_1, \ldots , s_N\}$. The binary variable $s_i$ indicates activating by 1 and deactivating by 0. Then we have
\begin{corollary}
The activating state in R\textsuperscript{3}eLU-forward is $\epsilon_p$-DP.
\label{cor:step2}
\end{corollary}
\begin{proof}
\normalfont
Supposing the state of $\hat{\bm{v}}^g$ is $\bm{s}_0$, containing $K$ 1s, and $\bm{h}$ contains indexes of top-K elements, then the probability of observing an activation state $s$ for a given original activation states $\bm{s}_0$ should be
\begin{equation}
    \Pr[\bm{s} \vert \bm{s}_0] = \frac{1}{2^{N-K}} \prod \limits_{i\in \bm{h}} p_i^{s_i} \cdot (1-p_i)^{1-s_i}.
\end{equation}
Now, we consider any two arbitrary states $\bm{s}_1$ and $\bm{s}_2$. The difference in activation states can be bounded by
\begin{equation}
\begin{split}
    \frac{\Pr[\bm{s} \vert \bm{s}_1]}{\Pr[\bm{s} \vert \bm{s}_2]} & = 
    \frac{\frac{1}{2^{N-K}} \prod \limits_{i\in \bm{h_1}} p_i^{s_i} \cdot (1-p_i)^{1-s_i}}{\frac{1}{2^{N-K}} \prod \limits_{i\in \bm{h_2}} p_i^{s_i} \cdot (1-p_i)^{1-s_i}} \\
    & = (\frac{\max\limits_{i\in \bm{h_1}}\{p_i\}}{1-\max\limits_{i\in \bm{h_2}}\{p_i\}})^K  = (\frac{e^{\frac{\epsilon_p}{K}} \frac{1}{1 + e^{\epsilon_p/K}}}{\frac{1}{1 + e^{\epsilon_p/K}}})^K  = e^{\epsilon_p}
\end{split}
\end{equation}
\end{proof}

\begin{corollary}
Given privacy budgets $\epsilon_p$ and $\epsilon_l$ for randomized response and Laplace mechanism respectively, the output of R\textsuperscript{3}eLU-forward procedure is ($\epsilon_p + \epsilon_l$)-DP.
\label{cor:re3lu_for_single}
\end{corollary}
\begin{proof}
\normalfont
Given activation state $\bm{s}$ and sensitivity $\Delta \le 2KC$,
\begin{equation}
    \frac{\Pr[\bm{\tilde{a}}^g \vert \bm{\hat{v}}, \bm{s}]}{\Pr[\bm{\tilde{a}}^g \vert \bm{\hat{v}}^{\prime}, \bm{s}]}  \le e^{\frac{\Delta}{\frac{2KC}{\epsilon_l}}} 
    \le e^{\epsilon_l}.
\end{equation}
Then, the difference of R\textsuperscript{3}eLU-forward outputs for any arbitrary $\bm{v}$ and $\bm{v}^{\prime}$ can be bounded by
\begin{equation}
\begin{split}
    \frac{\Pr[\bm{\tilde{a}}^g \vert \bm{v}]}{\Pr[\bm{\tilde{a}}^g \vert \bm{v}^{\prime}]} & = 
    \frac{\Pr[\bm{\tilde{a}}^g \vert \bm{\hat{v}}, \bm{s}] \cdot \Pr[\bm{s}\vert \bm{s}_{\bm{a}}]}{\Pr[\bm{\tilde{a}}^g \vert \bm{\hat{v}}^{\prime}, \bm{s}] \cdot \Pr[\bm{s}\vert \bm{s}_{\bm{a}^{\prime}}]} \\
    & \le e^{\epsilon_l} \cdot e^{\epsilon_p} \\
    & = e^{\epsilon_l + \epsilon_p}
\end{split}
\end{equation}
\end{proof}

Please note that $\epsilon = \epsilon_l + \epsilon_p$ is for the whole training dataset. By following the privacy amplification theory, each training step is $\gamma \epsilon$-DP, where $\gamma$ is the sampling ratio of a data batch. Now we can give the total privacy budget of the entire training process for the guest using the strong composition theorem.
\begin{corollary}
Split learning for the guest with R\textsuperscript{3}eLU-forward achieves $(\epsilon_g, \delta_g)$-DP , where $\epsilon_g = \gamma \epsilon \sqrt{2T\ln(\frac{1}{\delta_g})} + \gamma \epsilon T(e^{\gamma \epsilon}-1)$.
\end{corollary}
Since we construct R\textsuperscript{3}eLU-forward and R\textsuperscript{3}eLU-backward using the same method, these two procedures have the same analysis result if we constrain that they use the same input. In this way, we can conclude that the output of R\textsuperscript{3}eLU-backward procedure is ($\epsilon_p + \epsilon_l$)-DP, where $\epsilon_p$ and $\epsilon_l$ are budgets for randomized response and Laplace mechanism, respectively.
\begin{corollary}
Split learning for the host with R\textsuperscript{3}eLU-backward achieves $(\epsilon_h, \delta_h)$-DP, where $\epsilon_h = \gamma \epsilon \sqrt{2T\ln(\frac{1}{\delta_h})} + \gamma \epsilon T(e^{\gamma \epsilon}-1)$.
\end{corollary}

We remark that the dynamic privacy budget allocation improves budget utilization without any additional privacy loss. In other words, an important feature gets a higher probability of retaining the activation state than an insignificant feature, which leads to a larger privacy budget. But the total privacy budget of all features remains unchanged.


\begin{corollary}
When the guest runs R\textsuperscript{3}eLU-forward procedure with the dynamic privacy budget allocation, the output is still $\epsilon$-DP, $\epsilon=\epsilon_p + \epsilon_l$.
\label{cor:re3lu_for_dynamic}
\end{corollary}
\begin{proof}
\normalfont
Since activation values are clipped with constant $C$, $\max_{i\in \bm{h}}\{p_i\}$ will not be affected. Thus, the activation state with the dynamic privacy budget allocation is still $\epsilon_p$-DP. Given the activation state $\bm{s}$ and sensitivity $\Delta$, then
\begin{equation}
    \frac{\Pr[\bm{\tilde{a}}^g \vert \hat{\bm{v}}, \bm{s}]}{\Pr[\bm{\tilde{a}}^g \vert \hat{\bm{v}}^{\prime}, \bm{s}]}  \le e^{ \sum_{i=1}^{N}{(\frac{C \sum_{j=1}^{N}{U_j}}{\epsilon_l \Delta U_i})}
    } 
    \le e^{\sum_{i=1}^{N}{(\frac{\epsilon_l U_i}{\sum_{j=1}^{N}{U_j}}})}
    = e^{\epsilon_l}
\end{equation}
For any arbitrary $\bm{v}$ and $\bm{v}^{\prime}$, the difference of R\textsuperscript{3}eLU-forward outputs can be bounded by
\begin{equation}
\begin{split}
    \frac{\Pr[\bm{\tilde{a}}^g \vert \bm{v}]}{\Pr[\bm{\tilde{a}}^g \vert \bm{v}^{\prime}]} & = 
    \frac{\Pr[\bm{\tilde{a}}^g \vert \bm{\hat{v}}, \bm{s}] \cdot \Pr[\bm{s}\vert \bm{s_{a}}]}{\Pr[\bm{\tilde{a}}^g \vert \bm{\hat{v}}^{\prime}, \bm{s}] \cdot \Pr[\bm{s}\vert \bm{s}_{\bm{a}^{\prime}}]}  \le e^{\epsilon_l} \cdot e^{\epsilon_p} = e^{\epsilon_l + \epsilon_p}
\end{split}
\end{equation}
\end{proof}

\section{Evaluation}
We evaluate our privacy-preserving SplitNN solution from two aspects, model usability, and privacy loss. To be comprehensive, we will compare our solution with the baseline (without any protection) and the most relevant defensive solutions, i.e., a primitive Laplace mechanism \cite{dwork2014algorithmic} and DPSGD \cite{abadi2016deep}, the most well-known privacy-preserving deep learning solution, in the same setting. We will use the same fixed total privacy budget and the same split way (shown in the appendix) for all solutions. For the primitive Laplace mechanism, we simply add Laplacian noise to activations and partial losses to protect the privacy of the guest and host, respectively. For DPSGD, we add artificial noises to gradients of models on either side. For our solution, we set $\epsilon_p = \epsilon_l = \frac{\epsilon}{2}$ and $K$ as half number of features. We set $C$ as 10. We use the dynamic privacy budget allocation for features in our solution and set the initial importance as zero for all features.

All defensive solutions will be evaluated using three real-world datasets, MovieLens \cite{harper2015Movielens} and BookCrossing \cite{ziegler2005improving} for the recommendation, and MNIST \cite{lecun1998gradient} and CIFAR100 \cite{krizhevsky2009learning} for image classification. The \textit{MovieLens} 1-M dataset contains 1 million ratings of 4,000 movies collected from 6,000 users and users' demographic information such as gender and age. The \textit{BookCrossing} dataset includes 278,858 users' demographic information and 1,149,780 ratings of 271,379 books. The \textit{MNIST} database has 70,000 handwriting image examples of digital numbers from 0 to 9. The \textit{CIFAR100} database has 60,000 image examples for 100 classes. Each image example has one superclass as its rough label and one class as its accurate label. We will use batch size 32, a learning rate of 0.01, and an Adam optimizer as default. Since different datasets and defensive solutions may require various epochs for split learning, we will compare the metrics when the learning converges or the privacy budget is drained. All experimental results are averaged across multiple runs.

Before the evaluation, we verify the feasibility of our dynamic importance estimation method. During the verification, we observe the importance estimation of the neurons in the cut layer. The importance estimation of each neuron is calculated as Eq. (\ref{equ:9}). By accumulating all intermediate results of the importance estimated for neurons, we find that the final importance is almost the same as obtained by the original estimation method on a well-trained model's final state. The importance estimation results of neurons are shown in Figure~\ref{fig:importance_compare}, proving the correctness of our dynamic importance estimation method and the existence of unbalanced feature importance. As can be seen, the similarity between the accumulated (dynamic) estimation and the original (stable) estimation indicates that our accumulated approach to estimating the importance of a neuron works as well as the original approach.
\begin{figure}[ht]
    \centering
    \includegraphics[scale=0.35]{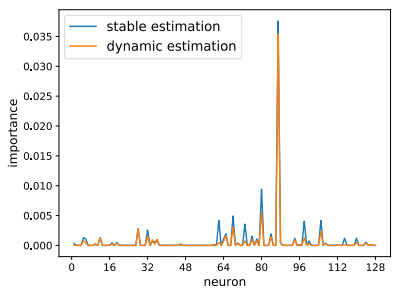}
    \caption{Estimation results of neuron importance.}
    \label{fig:importance_compare}
\end{figure}


\begin{table}[ht]
\centering
\caption{Model usability results while preserving the privacy of the guest.}
\label{tab:guest_usability}
\scriptsize
\begin{tabular}{c|ccc|ccc|ccc|ccc}
\hline
\multirow{2}{*}{$\epsilon$} & \multicolumn{3}{c|}{MovieLens} & \multicolumn{3}{c|}{BookCrossing} & \multicolumn{3}{c|}{MNIST} & \multicolumn{3}{c}{CIFAR100} \\ \cline{2-13} 
                            & Laplace    & DPSGD    & \textbf{Ours}   & Laplace     & DPSGD     & \textbf{Ours}    & Laplace  & DPSGD  & \textbf{Ours}  & Laplace    & DPSGD & \textbf{Ours}  \\ \hline
0.1                         &30.84\%            &32.29\%          &\textbf{34.03\%}        &57.02\%             &55.89\%           &\textbf{58.18\%}         &17.43\%          &30.21\%        &\textbf{32.41\%}       & 20.25\% & 34.21\% & \textbf{37.87\%}\\
0.5                         &41.25\%            &43.69\%          &\textbf{43.87\%}        &57.67\%             &56.14\%           &\textbf{58.54\%}         &27.33\%          &58.43\%        &\textbf{60.38\%}    & 39.74\% & 51.36\% & \textbf{58.22\%}   \\
1.0                         &48.16\%            &49.09\%          &\textbf{50.56\%}        &58.02\%             &56.56\%           &\textbf{58.42\%}         &31.05\%          &75.58\%        &\textbf{76.60\%}     & 46.19\% & 60.48\% &  \textbf{65.32\%} \\
2.0                         &49.32\%            &50.38\%          &\textbf{50.49\%}        &58.74\%             &56.91\%           &\textbf{59.24\%}         &38.92\%          &92.90\%        &\textbf{93.53\%}   & 56.30\% & 69.72\% &  \textbf{73.35\%}  \\
4.0                         &49.26\%            &\textbf{50.86\%}          &50.73\%        &59.01\%             &57.16\%           &\textbf{59.26\%}         &95.37\%          &\textbf{95.87\%}        &94.12\%   & 57.04\% & 70.86\% &  \textbf{74.41\%}  \\ \hline
\end{tabular}

\end{table}

\begin{table}[ht]
\centering
\caption{Model usability results while preserving the privacy of the host.}
\label{tab:host_usability}
\scriptsize

\begin{tabular}{c|ccc|ccc|ccc|ccc}
\hline
\multirow{2}{*}{$\epsilon$} & \multicolumn{3}{c|}{MovieLens} & \multicolumn{3}{c|}{BookCrossing} & \multicolumn{3}{c|}{MNIST} & \multicolumn{3}{c}{CIFAR100} \\ \cline{2-13} 
                            & Laplace    & DPSGD    & \textbf{Ours}   & Laplace     & DPSGD     & \textbf{Ours}    & Laplace  & DPSGD  & \textbf{Ours}  & Laplace    & DPSGD & \textbf{Ours} \\ \hline
0.1                         &31.47\%            &30.68\%          &\textbf{33.98\%}        &57.37\%             &57.46\%           &\textbf{58.26\%}         &27.64\%          &\textbf{33.45\%}        &32.36\%    & 17.04\% & 34.22\% & \textbf{38.96\%}  \\
0.5                         &41.75\%            &42.31\%          &\textbf{42.67\%}        &58.62\%             &58.24\%           &\textbf{58.59\%}         &55.38\%          &65.28\%        &\textbf{67.83\%}    & 25.82\% & 43.96\% & \textbf{53.72\%}   \\
1.0                         &47.43\%            &48.29\%          &\textbf{50.39\%}        &59.49\%             &58.44\%           &\textbf{59.77\%}         &71.95\%          &\textbf{89.74\%}        &88.14\%    & 37.69\% & 55.28\% & \textbf{61.48\%}   \\
2.0                         &49.86\%            &50.43\%          &\textbf{51.47\%}        &59.34\%             &59.97\%           &\textbf{60.27\%}         &89.15\%          &\textbf{92.66\%}        &92.52\%     & 51.87\% & 65.67\% & \textbf{69.60\%}  \\
4.0                         &49.57\%            &50.09\%          &\textbf{51.62\%}        &59.55\%             &\textbf{60.75\%}           &60.66\%         &94.61\%          &\textbf{95.37\%}        &95.01\%    & 51.87\%  & 66.89\% & \textbf{70.70\%}  \\ \hline
\end{tabular}
\end{table}

\subsection{Model Usability}
Since artificial perturbation may affect the learning procedure, we evaluate how SplitNN is affected by privacy-preserving solutions. Two asymmetric parties of SplitNN may have different influences on learning. Thus, we will evaluate model usability concerning privacy from the perspective of the guest or the host, respectively. We use an averaged test accuracy across all test samples for the evaluation of model usability. Precisely, the test accuracy of a recommendation model is calculated using a top-10 hit ratio, while the test accuracy of an image classifier is its prediction accuracy. In Table~\ref{tab:guest_usability} and Table~\ref{tab:host_usability}, we show the model usability results regarding various privacy budget values of the two parties. We note that model accuracy baselines of Movielens, BookCrossing, MNIST and CIFAR100 for SplitNN are 56.62\%, 61.70\%, 98.00\% and 76.20\%, respectively.

For the MovieLens model, our solution achieves the best model usability in most cases, especially with a smaller privacy budget. DPSGD has a better result when $\epsilon=4$ for the guest. But a significant privacy leakage will be caused in this case. For the BookCrossing model, the model usability of our solution is relatively high in cases of protecting the guest and the host. Similarly, DPSGD achieves a better result when $\epsilon=4$ by sacrificing the host's privacy. Results show some differences for the MNIST model. DPSGD has better results when protecting the host's privacy. The reason is that split learning for an image classification model segments image samples roughly, making our dynamic budget allocation approach malfunction. Meanwhile, DPSGD is not designed to protect partial loss in SplitNN, leading to an optimistic estimation of the threat against the host. On the contrary, our solution has a competitive performance in image classification. For the CIFAR100 model, our method outperforms other protection mechanisms. As the host has the major part of the model, the accuracy drops 5.5\% for host protection while only 1.79\% for guest protection. These results show that our method with dynamic privacy budget allocation can allocate appropriate privacy budget on different neurons and achieve high accuracy even on complex datasets and models while the primitive Laplace mechanism suffers a 24-percentage points drop in accuracy for CIFAR100 evaluation due to its indiscrimination on all neurons, as shown in Table~\ref{tab:host_usability}.

\subsection{Privacy Preservation}
We evaluate the performance of privacy preservation by comparing attack results against SplitNN with and without the defense. We will mount property inference and data reconstruction attacks against the guest and the host, respectively. The prediction accuracy of the adversary's inference model will be used to measure the performance of the property inference attack. As for the data reconstruction attack, the adversary tries to generate data samples as similar as possible to the target's private data. In this case, a mean squared error (MSE) between a generated sample and a target data sample is commonly used for the adversarial effect measurement.

\begin{table}[ht]
\centering
\caption{Results of defending the guest against property inference attack.}
\label{tab:guest_defense_inference}
\scriptsize

\begin{tabular}{c|ccc|ccc|ccc|ccc}
\hline
\multirow{2}{*}{$\epsilon$} & \multicolumn{3}{c|}{MovieLens} & \multicolumn{3}{c|}{BookCrossing} & \multicolumn{3}{c}{MNIST} & \multicolumn{3}{c}{CIFAR100}\\ \cline{2-13} 
                            & Laplace    & DPSGD    & \textbf{Ours}   & Laplace     & DPSGD     & \textbf{Ours}    & Laplace  & DPSGD  & \textbf{Ours}   & Laplace  & DPSGD  & \textbf{Ours}\\ \hline
0.1                         &66.99\%            &77.71\%          &\textbf{60.99\%}        &\textbf{54.76\%}             &73.29\%           &55.78\%         &\textbf{43.27\%}          &53.95\%        &44.33\%       &\textbf{50.76\%}            &79.14\%          &53.97\%\\
0.5                         &66.16\%            &74.23\%          &\textbf{64.16\%}        &\textbf{54.97\%}             &74.52\%           &56.33\%         &46.92\%          &54.23\%        &\textbf{45.59\%}       &\textbf{51.47\%}            &79.26\%          &55.13\%\\
1.0                         &\textbf{67.19\%}            &78.65\%          &68.18\%        &\textbf{55.03\%}             &74.96\%           &58.65\%         &47.58\%          &54.26\%        &\textbf{47.51\%}       &\textbf{50.40\%}            &79.37\%          &55.72\%\\
2.0                         &68.65\%            &73.06\%          &\textbf{68.56\%}        &\textbf{54.85\%}             &74.26\%           &58.14\%         &\textbf{48.06\%}          &54.65\%        &52.87\%       &60.76\%            &79.37\%          &\textbf{58.81\%}\\
4.0                         &\textbf{69.14\%}            &76.18\%          &71.91\%        &\textbf{54.92\%}             &74.33\%           &60.76\%         &\textbf{48.47\%}          &54.57\%        &55.73\%       &60.81\%            &79.35\%          &\textbf{58.03\%}\\ \hline
\end{tabular}
\end{table}

\begin{table}[ht]
\centering
\caption{Results of defending the host against property inference attack.}
\label{tab:host_defense_inference}
\scriptsize

\begin{tabular}{c|ccc|ccc|ccc|ccc}
\hline
\multirow{2}{*}{$\epsilon$} & \multicolumn{3}{c|}{MovieLens} & \multicolumn{3}{c|}{BookCrossing} & \multicolumn{3}{c}{MNIST} & \multicolumn{3}{c}{CIFAR100}\\ \cline{2-13} 
                            & Laplace    & DPSGD    & \textbf{Ours}   & Laplace     & DPSGD     & \textbf{Ours}    & Laplace  & DPSGD  & \textbf{Ours}  & Laplace  & DPSGD  & \textbf{Ours}\\ \hline
0.1                         &  53.46\%          &  78.59\%        &\textbf{51.86\%}    &  \textbf{54.55\%}      &      74.35\%    & 59.42\%     &60.34\%         &80.29\%         &\textbf{48.74\%}       &50.42\%             &51.46\%           &\textbf{41.89\%}\\
0.5                         &   53.46\%     &    75.64\%      &\textbf{51.89\%}       &        \textbf{54.62\%}     &       74.36\%    &  59.42\%   & 59.82\%         &81.92\%         &\textbf{49.71\%}     &50.38\%             &52.23\%           &\textbf{44.26\%} \\
1.0                         &   53.46\%    &    73.54\%  &  \textbf{52.75\%}      &        \textbf{54.95\%}    &  74.39\%     &      59.52\%   & 59.74\%         &82.80\%       &\textbf{50.48\%}        &49.95\%             &51.95\%           &\textbf{44.78\%}\\
2.0                         &   \textbf{53.47\%}  &       75.05\% &   59.77\%     &       \textbf{54.40\%}     &      74.39\%    &    58.13\%    & 60.38\%         & 88.88\%        & \textbf{50.57\%}       &50.77\%             &51.67\%           &\textbf{50.16\%}\\ 
4.0                     & \textbf{53.48\%}   &     79.28\% &    56.52\%    &     \textbf{54.95\%}     &  74.39\%     &    62.04\%   & 60.62\%        &89.73\%        &\textbf{50.47\%}       &51.52\%             &51.76\%           &\textbf{51.27\%}\\ \hline
\end{tabular}
\end{table}

\subsubsection{Defense against property inference attack}
A property inference attack is to infer an existing property (or attribute) of data samples. For example, an adversarial host in our experiments infers the age attribute of the guest's data for a recommendation model. However, the host has no idea of the age distribution since training data is vertically partitioned. We carry out the same property infence attack as \cite{luo2021feature}. We give evaluation results of the defensive effect of the guest and the host in Table~\ref{tab:guest_defense_inference} and Table~\ref{tab:host_defense_inference}, respectively. We use the prediction accuracy of the adversary's inference model as a criterion for evaluation. The higher the prediction accuracy, the more probable success the property inference attack may achieve. In other words, the worse the defensive effect is. As for an image classification model, an unknown patch of image samples will be inferred. The attack accuracy against baselines of MovieLens, BookCrossing, MNIST and CIFAR100 models can achieve above 80\%, 79\%, 94\% and 87\% by an adversarial host, 80\%, 78\%, 57\% and 53\% by an adversarial guest, respectively. However, our solution can effectively mitigate the adversarial effect during training and decrease the attack accuracy significantly. It should be noted that the primitive Laplace mechanism frustrates the inference attack because the artificial noise added by the Laplace mechanism is indiscriminate, leading to conspicuous damage to the model's usability. Even so, our solution has significant advantages on MovieLens and MNIST datasets. In contrast, the primitive Laplace mechanism cannot protect image classification models, while DPSGD cannot defeat the attack. On the BookCrossing dataset, the primitive Laplace mechanism seems to have a better performance. We infer that the simplicity of the BookCrossing dataset and its corresponding model may lead to this situation. As the noise generated by the primitive Laplace mechanism is haphazard, the model may fail to learn this certain property that the property inference attack aims for. As a result, the property inference attack performs worse on the primitive Laplace mechanism while it maintains a good level of accuracy due to its simplicity. For the CIFAR100 model, While protecting the guest, it seems that at a low privacy budget, the primitive Laplace is better. However, we remind that the model usability is extremely low when applying the primitive Laplace mechanism. Under other situations, our method can effectively decrease the effect of a property inference attack, especially at a low privacy budget.

\begin{table}[ht]
\centering
\caption{Results of defending the guest against data reconstruction attack.}
\label{tab:guest_defense_reconstruction}
\scriptsize

\begin{tabular}{c|ccc|ccc|ccc|ccc}
\hline
\multirow{2}{*}{$\epsilon$} & \multicolumn{3}{c|}{MovieLens} & \multicolumn{3}{c|}{BookCrossing} & \multicolumn{3}{c}{MNIST} & \multicolumn{3}{c}{CIFAR100}\\ \cline{2-13} 
                            & Laplace    & DPSGD    & \textbf{Ours}   & Laplace     & DPSGD     & \textbf{Ours}    & Laplace  & DPSGD  & \textbf{Ours} & Laplace  & DPSGD  & \textbf{Ours} \\ \hline
0.1                         &  0.2459          &  0.2455        &  \textbf{0.3223}      &   0.3216          &        0.2907   &    \textbf{0.3329}    &   1.8849       &    1.8885    & \textbf{2.0181}  &\textbf{12.5145}            &2.8622          &3.6983\\
0.5                         &   0.2453         &  0.2451        &  \textbf{0.3222}     &   0.3202      &  0.2902    &  \textbf{0.3329}       &  1.8024        &       1.8137 &    \textbf{1.9875}    &\textbf{12.5262}            &2.8537          &3.6891\\
1.0                         &   0.2453         &   0.2451       &  \textbf{0.3222}     & 0.3202     &    0.2902   &    \textbf{0.3221}     &    1.7857    &   1.7509     &   \textbf{1.9533}  &3.6419             &2.8351          &\textbf{3.6624}\\
2.0                         &   0.2452        &  0.2451    & \textbf{0.3222}       &  0.3202      &   0.2902     &       \textbf{0.3221}  &    1.7336      &     1.7469  &  \textbf{1.9391}     &2.9453            &2.7998           &\textbf{3.6383}\\
4.0                         &   0.2452         &   0.2451       &    \textbf{0.3222}    &   0.3202     &     0.2902   &   \textbf{0.3221}    &   1.7014       &   1.7440     & \textbf{1.9206}   &2.9502           &2.7743          &\textbf{3.6365}\\ \hline
\end{tabular}
\end{table}

\begin{table}[ht]
\centering
\caption{Results of defending the host against data reconstruction attack.}
\label{tab:host_defense_reconstruction}
\scriptsize

\begin{tabular}{c|ccc|ccc|ccc|ccc}
\hline
\multirow{2}{*}{$\epsilon$} & \multicolumn{3}{c|}{MovieLens} & \multicolumn{3}{c|}{BookCrossing} & \multicolumn{3}{c}{MNIST} & \multicolumn{3}{c}{CIFAR100}\\ \cline{2-13} 
                            & Laplace    & DPSGD    & \textbf{Ours}   & Laplace     & DPSGD     & \textbf{Ours}    & Laplace  & DPSGD  & \textbf{Ours}  & Laplace  & DPSGD  & \textbf{Ours}\\ \hline
0.1                      &  0.4032    &    0.2417    &     \textbf{0.5486}   & 0.4237      &  0.2758     &  \textbf{0.5066}   &   1.2887      & 1.0875    &  \textbf{1.8257}     &\textbf{13.2849}             &6.0999           &6.5283\\
0.5                         &    0.4024         & 0.2419         &  \textbf{0.5357}      &    0.4222         &    0.2756       &   \textbf{0.5149}      &       1.2778   &   1.0685     &   \textbf{1.7758}    &\textbf{12.8057}             &5.9302           &6.3719\\
1.0                         &    0.4008     &    0.2422      &   \textbf{0.5285}     &    0.4217         &    0.2743       &   \textbf{0.5235}      &  1.2602        &  1.0422      & \textbf{1.7528} &\textbf{12.7936}             &5.9283           &6.3531\\
2.0                         &   0.3982         &  0.2421        &  \textbf{0.5083}      & 0.4214      &   0.2697      &      \textbf{0.5150}   &    1.2613     &    1.0333    &     \textbf{1.7334}  &6.0397                       &5.9256           &\textbf{6.3453}\\
4.0                         &   0.3960         &  0.2422        &   \textbf{0.4819}     &     0.4194        &    0.2683      &    \textbf{0.5046}    &  1.2549        &  0.9996      &   \textbf{1.7262}    &6.0143                       &5.9247           &\textbf{6.3396}\\ \hline
\end{tabular}
\end{table}

\subsubsection{Defense against data reconstruction attack}
We take advantage of GANs and use the same reconstruction attack as \cite{hitaj2017deep}. We show the defense results against an adversarial host and an adversarial guest in Table~\ref{tab:guest_defense_reconstruction} and Table~\ref{tab:host_defense_reconstruction}, respectively. We note that the MSE is measured after the attack model has been trained sufficiently in all cases. The MSEs measured for the attack against baselines of MovieLens, BookCrossing, MNIST and CIFAR100 models are 0.2412, 0.2629, 0.9612 and 2.6335 by an adversarial host, 0.2369, 0.2402, 1.6998 and 5.7534 by an adversarial guest, respectively. Please note that these attack results against the baselines are frustrating because the reconstruction attack is hard to succeed in the semi-honest setting. Meanwhile, data samples in the two recommendation datasets are similar and embedded with the same feature vectors. This leads to similar reconstruction results and similar MSEs because the reconstruction of structured data in MovieLens and BookCrossing largely depends on the embedding module. But we can still conclude from the results that our solution has a dominant performance in the defense against reconstruction attacks on either side. For the CIFAR100 model, when at a low privacy budget, the primitive Laplace mechanism sacrifices the prediction accuracy to reach a high difference. However, our method can maintain a satisfying prediction accuracy while protecting against the reconstruction attack.

\subsubsection{Defense against feature space hijacking attack (FSHA)}
Please note that property inference and data reconstruction attacks implemented in FSHA \cite{Pasquini2021UnleashingTT} hijack the learning objective, offering the adversary an advantage over the previous attacks we have evaluated. In this setting, the malicious attacker trains a generator using SplitNN as a discriminator during the learning process. And a gradient-scaling trick is used to train the generator in FSHA. The sample generating process is essential to FSHA, meaning inference attacks depend on the reconstruction in FSHA. Thus, we will focus on the evaluation of defense against reconstruction attacks. If the generating part fails, the inference attack will be impossible. Since FSHA is comprehensively evaluated using the MNIST dataset, we give defense results of the MNIST model here. In Figure~\ref{fig:recovered_guest_minist} and Figure~\ref{fig:recovered_host_minist}, we give the reconstruction results of FHSA mounted by an adversarial host and an adversarial guest against the target samples used in \cite{Pasquini2021UnleashingTT}, respectively. The second row of the two figures shows the results of FSHA against baselines. The following rows show that our solution can effectively preserve private data for both the guest and the host, even if the privacy budget is relaxed to 4. The results of our solution against FSHA using other datasets and the results of different budget values are in the appendix.
\begin{figure*}[ht]
    \begin{minipage}[t]{1\linewidth}
    \centering
    \includegraphics[width=1\linewidth]{defense_origin_minist.eps}
    \end{minipage}\par\vspace{-1.2ex}
    \begin{minipage}[t]{1\linewidth}
    \centering
    \includegraphics[width=1\linewidth]{defense_recovered_minist.eps}
    \end{minipage}\par\vspace{-1ex}
    \begin{minipage}[t]{1\linewidth}
    \centering
    \includegraphics[width=1\linewidth]{defense_re3lu_0.1_minist.eps}
    \end{minipage}\par\vspace{-1ex}
    \begin{minipage}[t]{1\linewidth}
    \centering
    \includegraphics[width=1\linewidth]{defense_re3lu_1.0_minist.eps}
    \end{minipage}\par\vspace{-1ex}
    \begin{minipage}[t]{1\linewidth}
    \centering
    \includegraphics[width=1\linewidth]{defense_re3lu_4.0_minist.eps}
    \end{minipage}\par\vspace{-1ex}
    \caption{Reconstruction results of FSHA against the guest's data in the first row. The following rows are attack results against the original SplitNN and our solution ($\epsilon=0.1, 1.0, 4.0$), respectively.}
    \label{fig:recovered_guest_minist}
\end{figure*}
\begin{figure*}[ht]
    \vspace{-1ex}
    \begin{minipage}[t]{1\linewidth}
    \centering
    \includegraphics[width=0.8\linewidth]{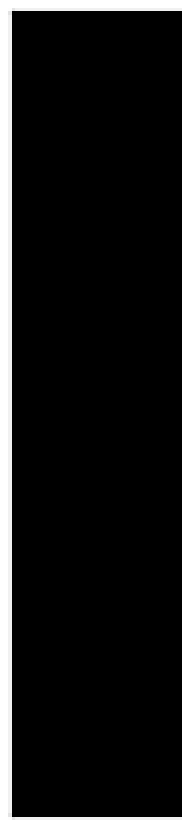}
    \end{minipage}\par\vspace{-1.2ex}
    \begin{minipage}[t]{1\linewidth}
    \centering
    \includegraphics[width=0.8\linewidth]{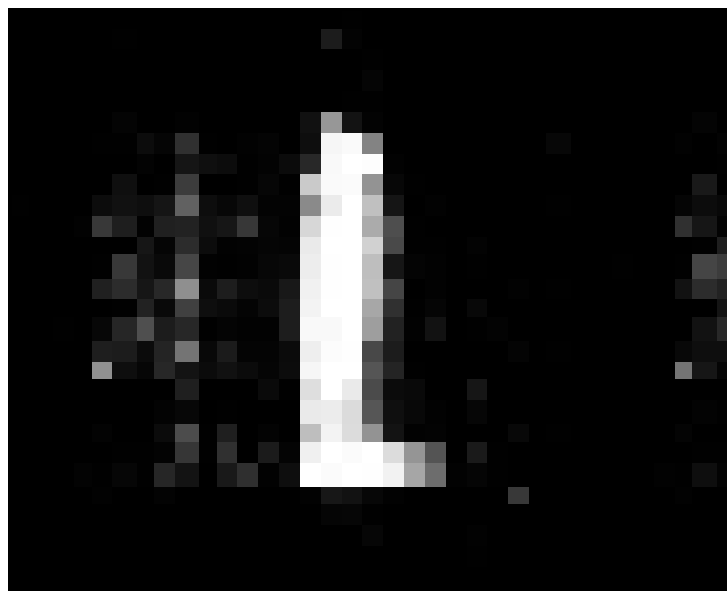}
    \end{minipage}\par\vspace{-1ex}
    \begin{minipage}[t]{1\linewidth}
    \centering
    \includegraphics[width=0.8\linewidth]{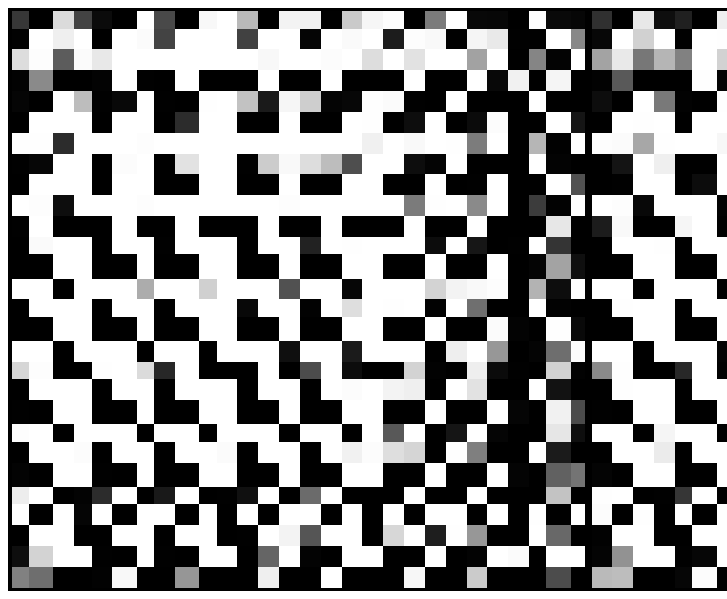}
    \end{minipage}\par\vspace{-1ex}
    \begin{minipage}[t]{1\linewidth}
    \centering
    \includegraphics[width=0.8\linewidth]{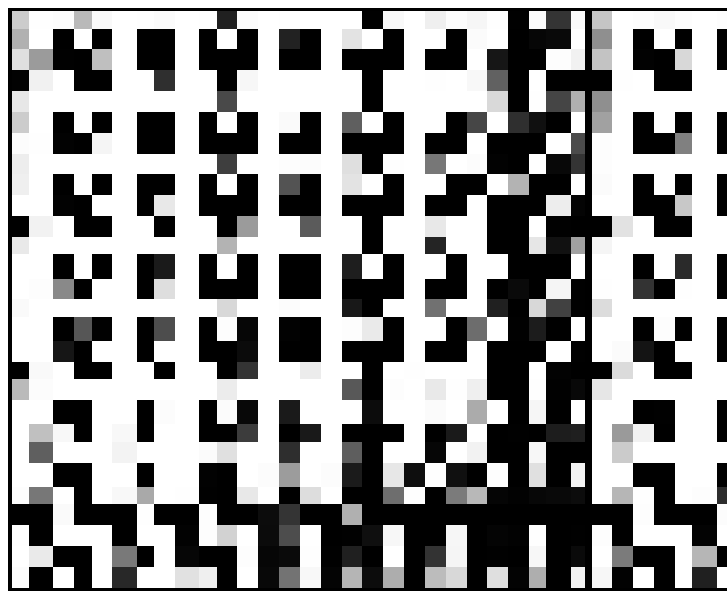}
    \end{minipage}\par\vspace{-1ex}
    \begin{minipage}[t]{1\linewidth}
    \centering
    \includegraphics[width=0.8\linewidth]{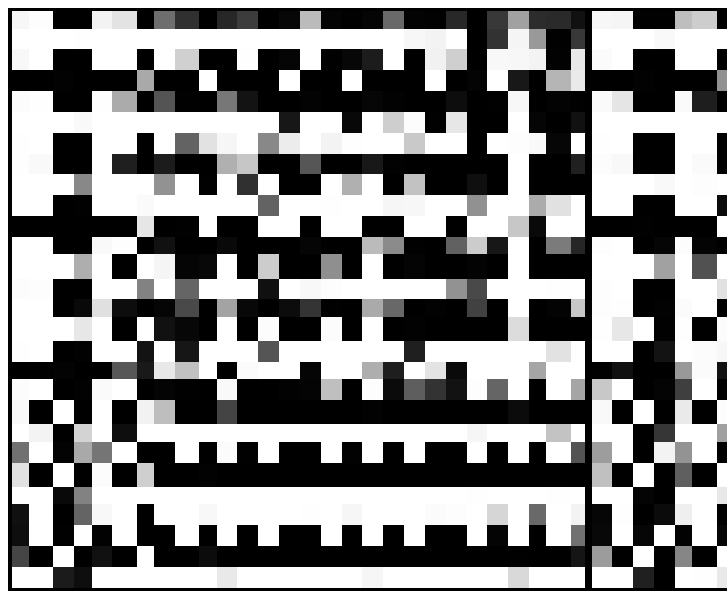}
    \end{minipage}\par\vspace{-1ex}
    \caption{Reconstruction results of FSHA against the host's data in the first row. The following rows are attack results against the original SplitNN and our solution ($\epsilon=0.1, 1.0, 4.0$), respectively.}
    \label{fig:recovered_host_minist}
\end{figure*}

\section{Conclusion}
We investigate privacy leakage issues in SplitNN for two parties. By mounting property inference, data reconstruction, and hijacking attacks, we confirm that both the host and the guest in SplitNN are under severe threats of privacy leakage. To mitigate the leakage, we design a new activation function R\textsuperscript{3}eLU and its derivative in a randomized-response manner. By integrating R\textsuperscript{3}eLU into SplitNN as an interacting tunnel, we implement R\textsuperscript{3}eLU-forward and R\textsuperscript{3}eLU-backward procedures. Through the privacy analysis, we confirm that SplitNN using R\textsuperscript{3}eLU-forward and R\textsuperscript{3}eLU-backward provides differential privacy for both two parties. Moreover, we propose a fine-grained privacy budget allocation scheme for assigning privacy budgets dynamically according to the parameters' importance. We finally conclude that our SplitNN solution outperforms the existing privacy-preserving solutions in model usability and privacy preservation through a comprehensive evaluation of different learning tasks. We also note that our solution concentrates on the leakage of private property and data samples. Other privacy issues, like label leakage in SplitNN, should be discussed separately. It is now unclear whether randomized-response solutions can deal with label leakage, which will be our work in the next.

\section*{Acknowledgement}
The authors would like to thank our shepherd Prof. Stjepan Picek and the anonymous reviewers for the time and effort they have kindly put into this paper. Our work has been improved by following the suggestions they have made. This work was supported in part by the Leading-edge Technology Program of Jiangsu-NSF under Grant BK20222001 and the National Natural Science Foundation of China under Grants NSFC-62272222, NSFC-61902176, NSFC-62272215.

\clearpage
\bibliographystyle{splncs04}
\bibliography{ppsplit}

\clearpage
\appendix
\section*{Appendix}

\subsection{Supplement Results of Evaluation}
To further investigate how our solution affects the learning process of SplitNN, we report learning results of a MovieLens recommendation model protecting the privacy of the guest and the host in Figure~\ref{fig:Guest Accuracy on MovieLens} and \ref{fig:Host accuray on Movielens}, respectively. In each plot, we show trends of training accuracy and testing accuracy as the training epoch increases. When $\epsilon=0.1$ for the guest or the host, model usability will be influenced seriously. In this case, the artificial perturbation is too large to maintain precise data characteristics. Things get much better when the privacy budget increases to 1 for either the guest or the host. Thus, we can conclude from the figures that our solution achieves satisfying model usability even with a small privacy budget for either side of SplitNN.

We notice that FSHA is sensitive to our privacy-preserving SplitNN solution. In Figure~\ref{fig:recovered_guest_minist} and Figure~\ref{fig:recovered_host_minist}, we find that neither an adversarial host nor an adversarial guest can reconstruct meaningful samples, even if the target's privacy budget is $\epsilon=4$. Therefore, we are curious about a practical choice of the privacy budget when dealing with FSHA. To this end, we give more defense results of our solution against FSHA using various privacy budget values in Table~\ref{tab:defense_fsha_host_budget}. We can conclude from the MSE results that two recommendation models prefer relatively low privacy budgets, such as $\epsilon=1.0$. However, it is interesting to see that the FSHA attack against an image classification model using SplitNN can be frustrated by our solution using a relatively high privacy budget, which also means a high model usability.
\begin{figure*}[ht]
    \centering
    \begin{minipage}{0.3\linewidth}
    \centering
    \includegraphics[scale=0.28]{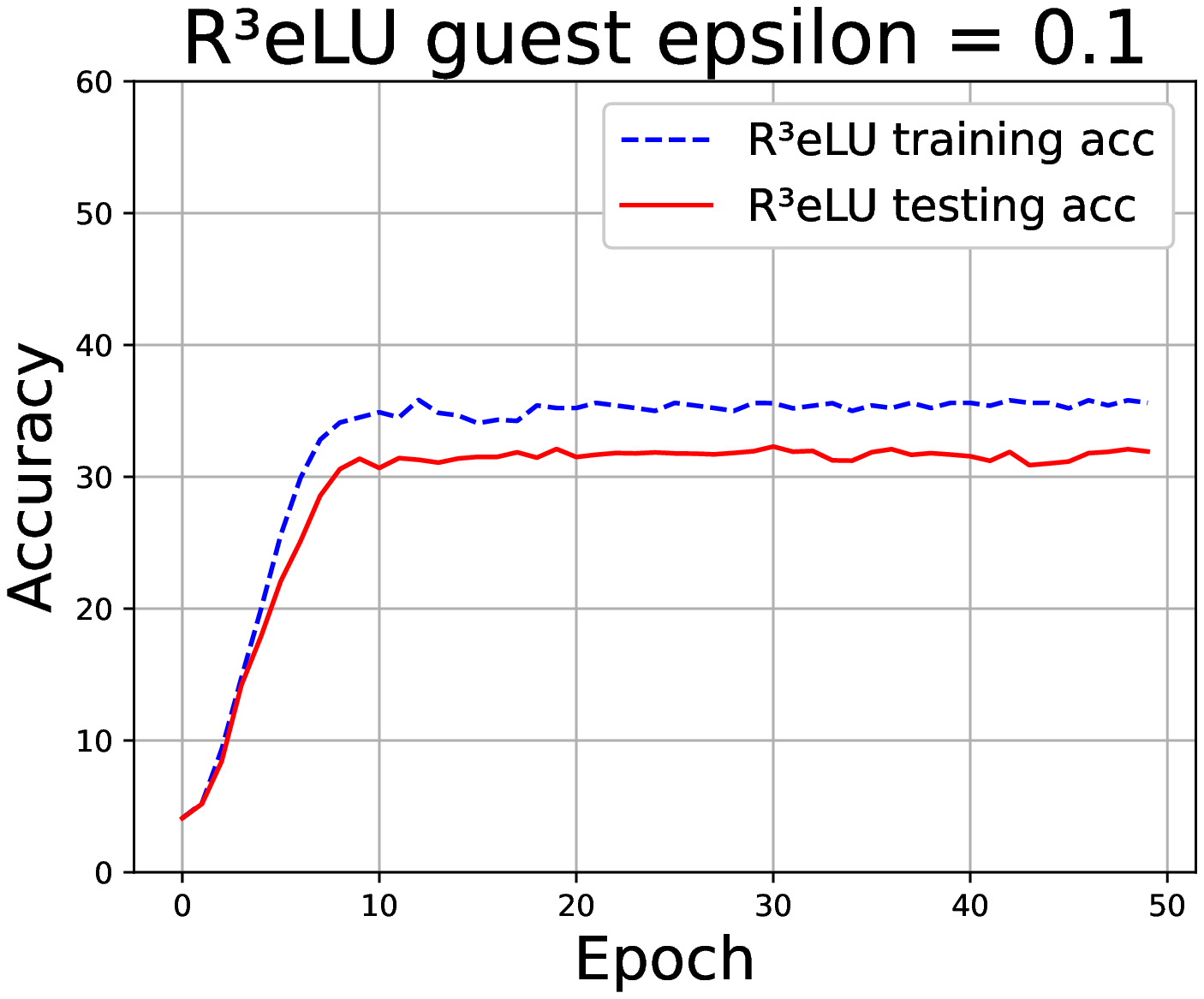}
    \end{minipage}
    \hfill
    \begin{minipage}{0.3\linewidth}
    \centering
    \includegraphics[scale=0.28]{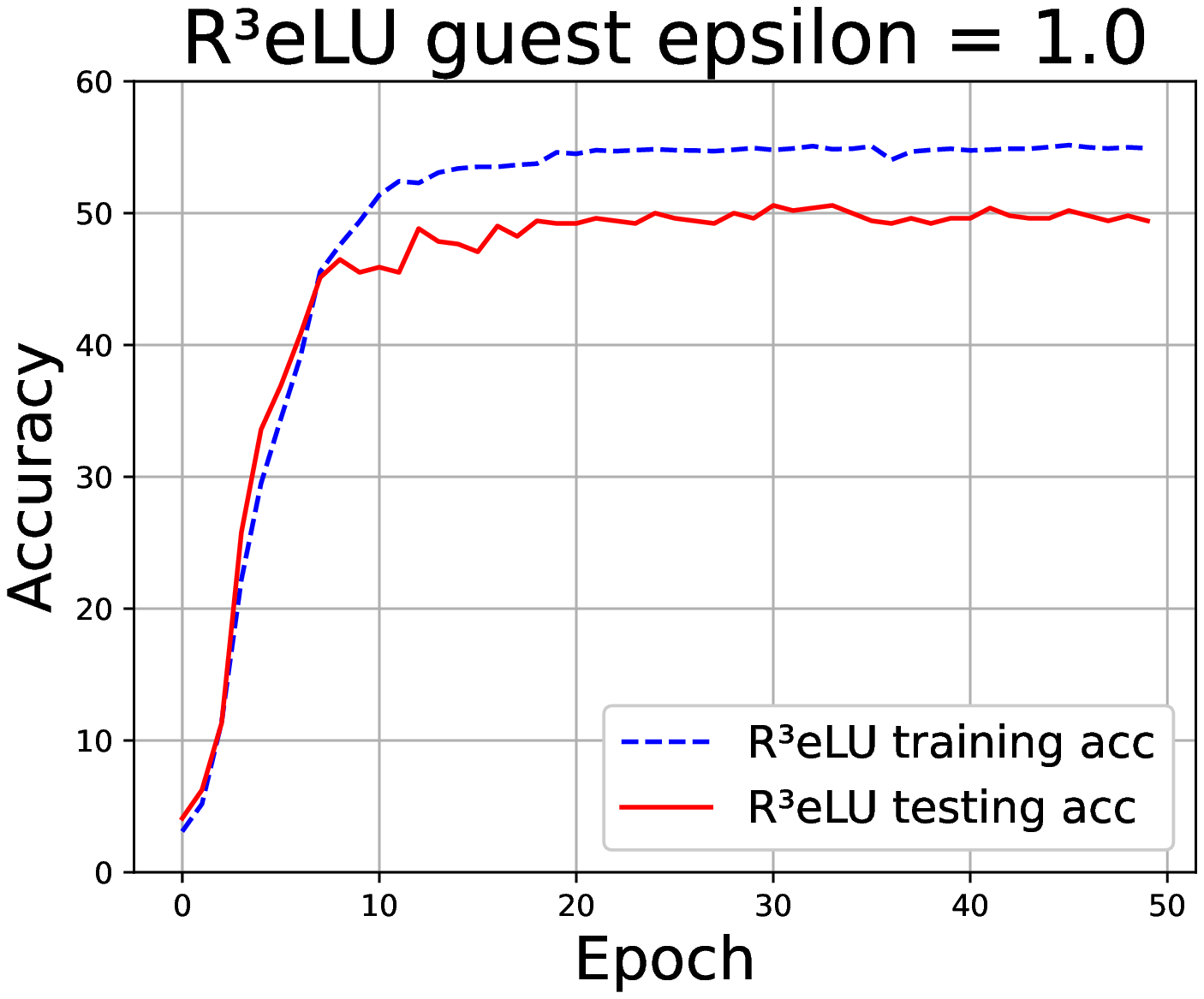}
    \end{minipage}
    \hfill
    \begin{minipage}{0.3\linewidth}
    \centering
    \includegraphics[scale=0.28]{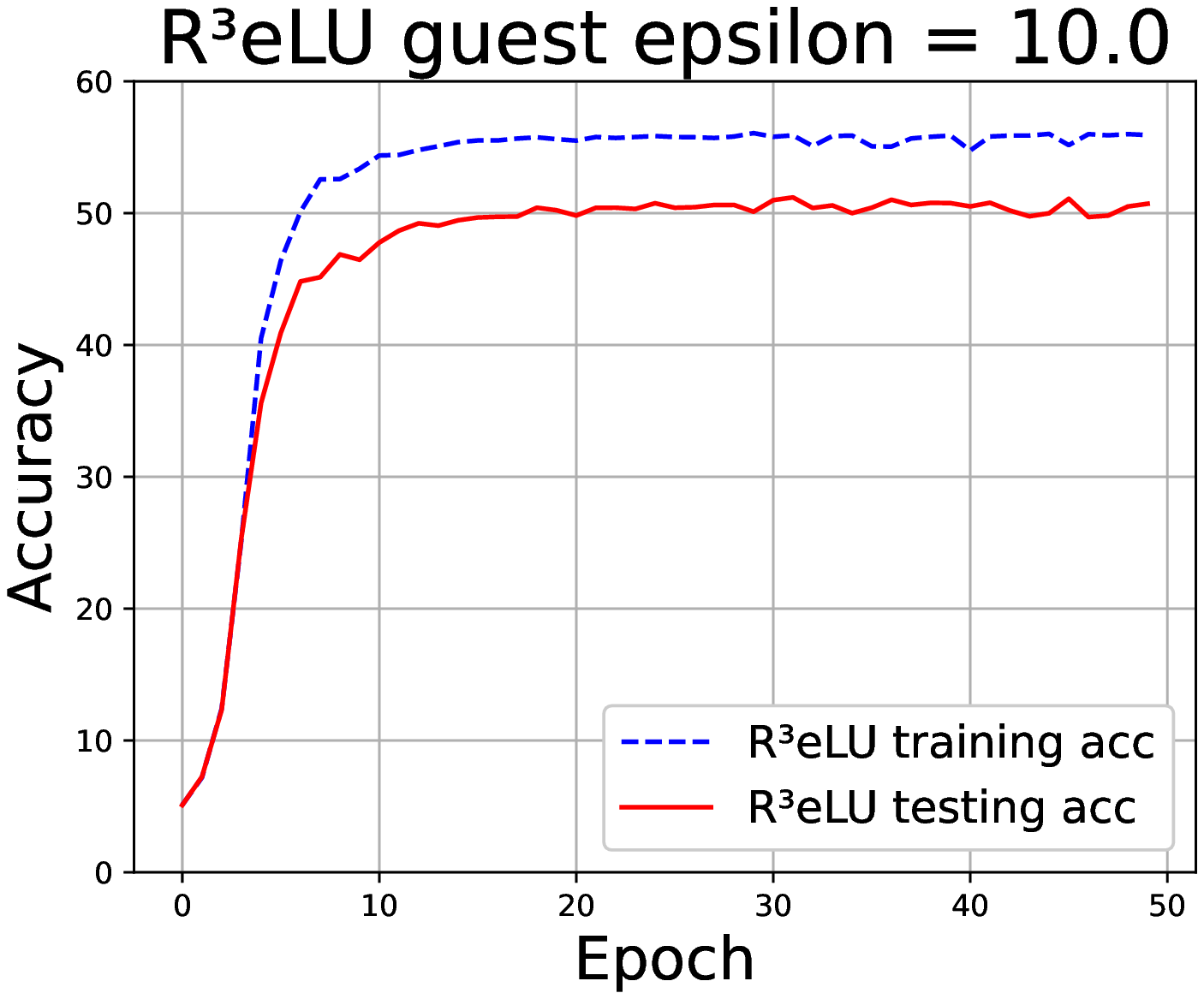}
    \end{minipage}
    \caption{SplitNN learning curve with the guest's privacy protected by our solution.}
    \label{fig:Guest Accuracy on MovieLens}
\end{figure*}
\begin{figure*}[ht]
    \centering
    \begin{minipage}{0.3\linewidth}
    \centering
    \includegraphics[scale=0.28]{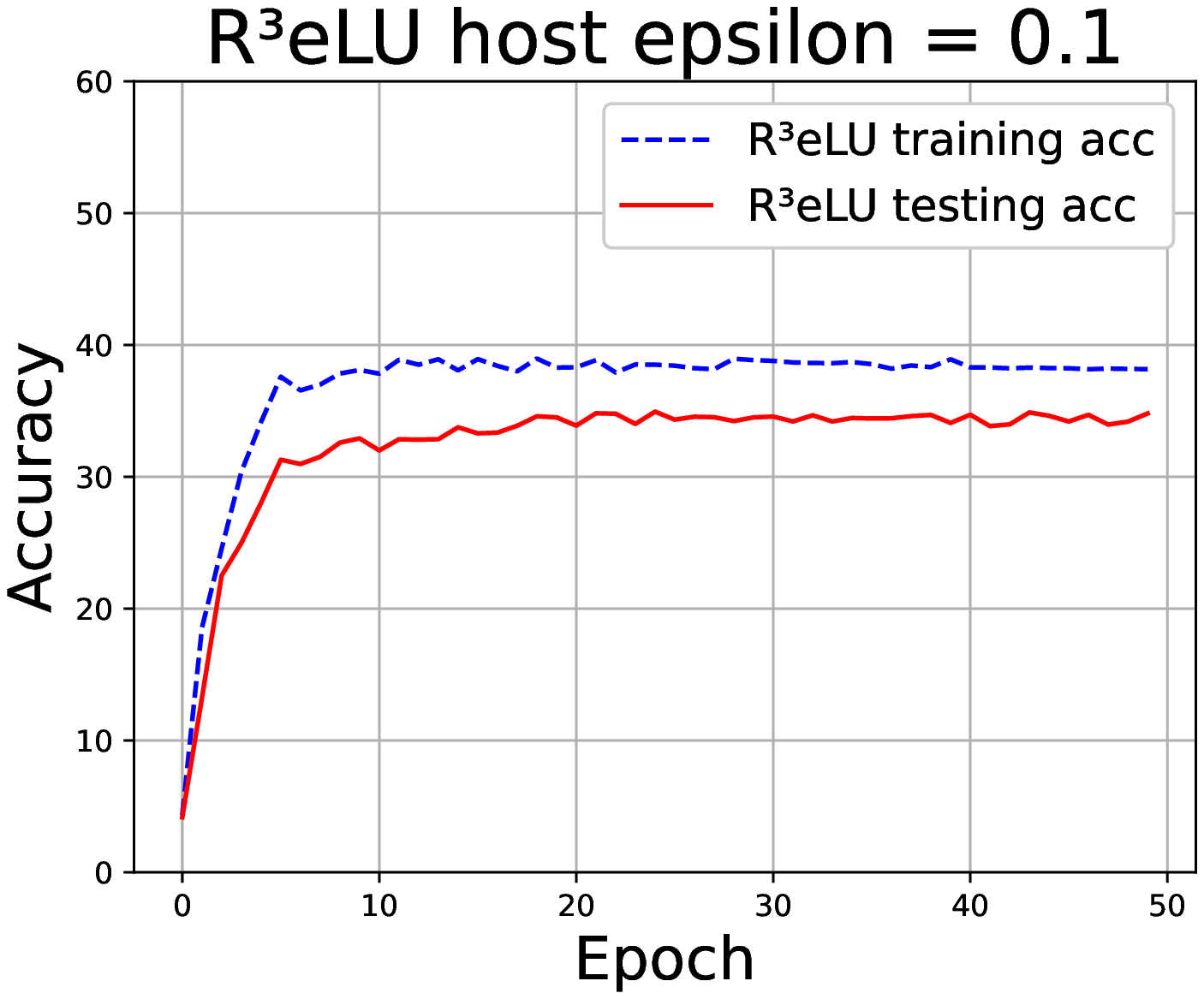}
    \end{minipage}
    \hfill
    \begin{minipage}{0.3\linewidth}
    \centering
    \includegraphics[scale=0.28]{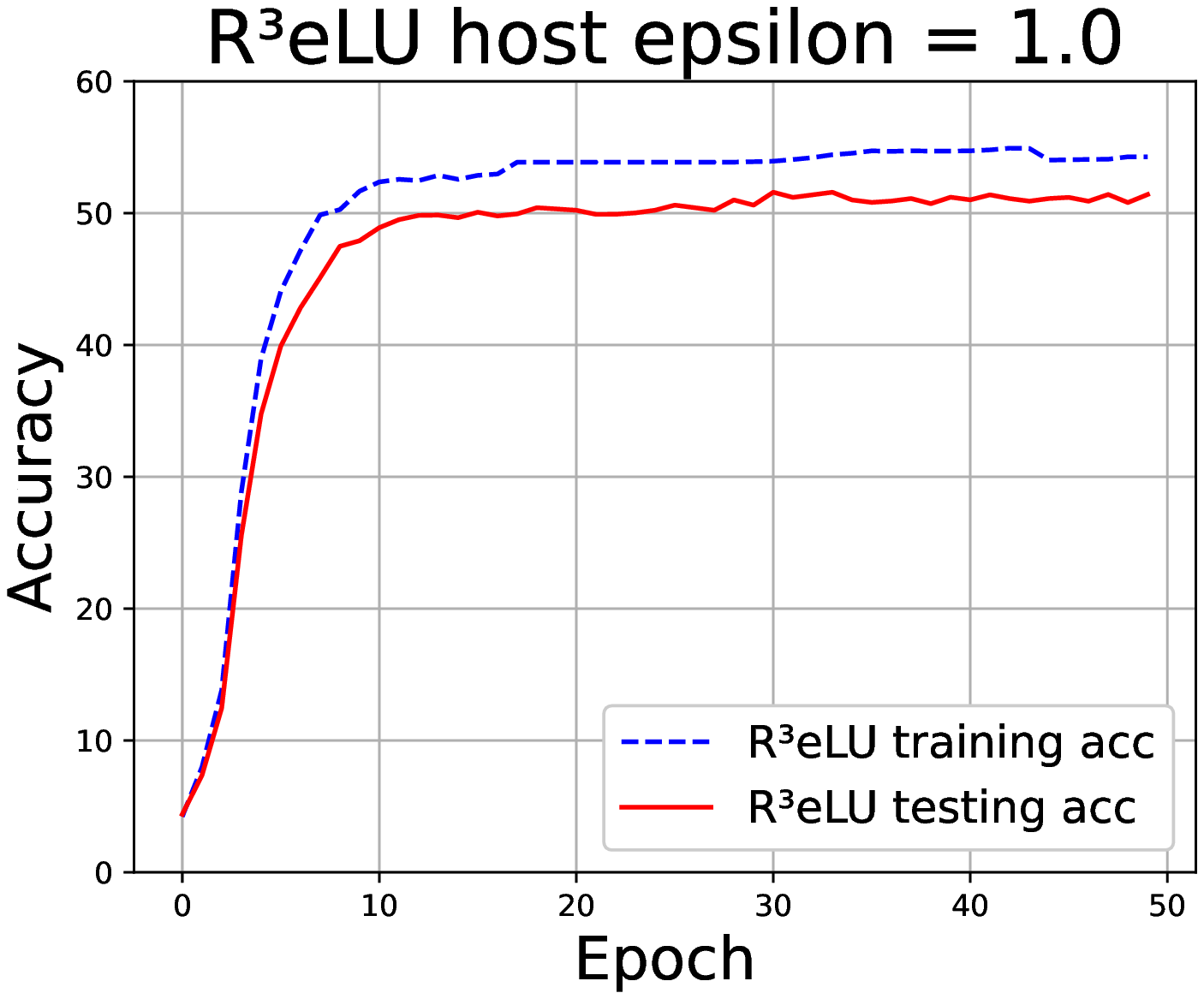}
    \end{minipage}
    \hfill
    \begin{minipage}{0.3\linewidth}
    \centering
    \includegraphics[scale=0.28]{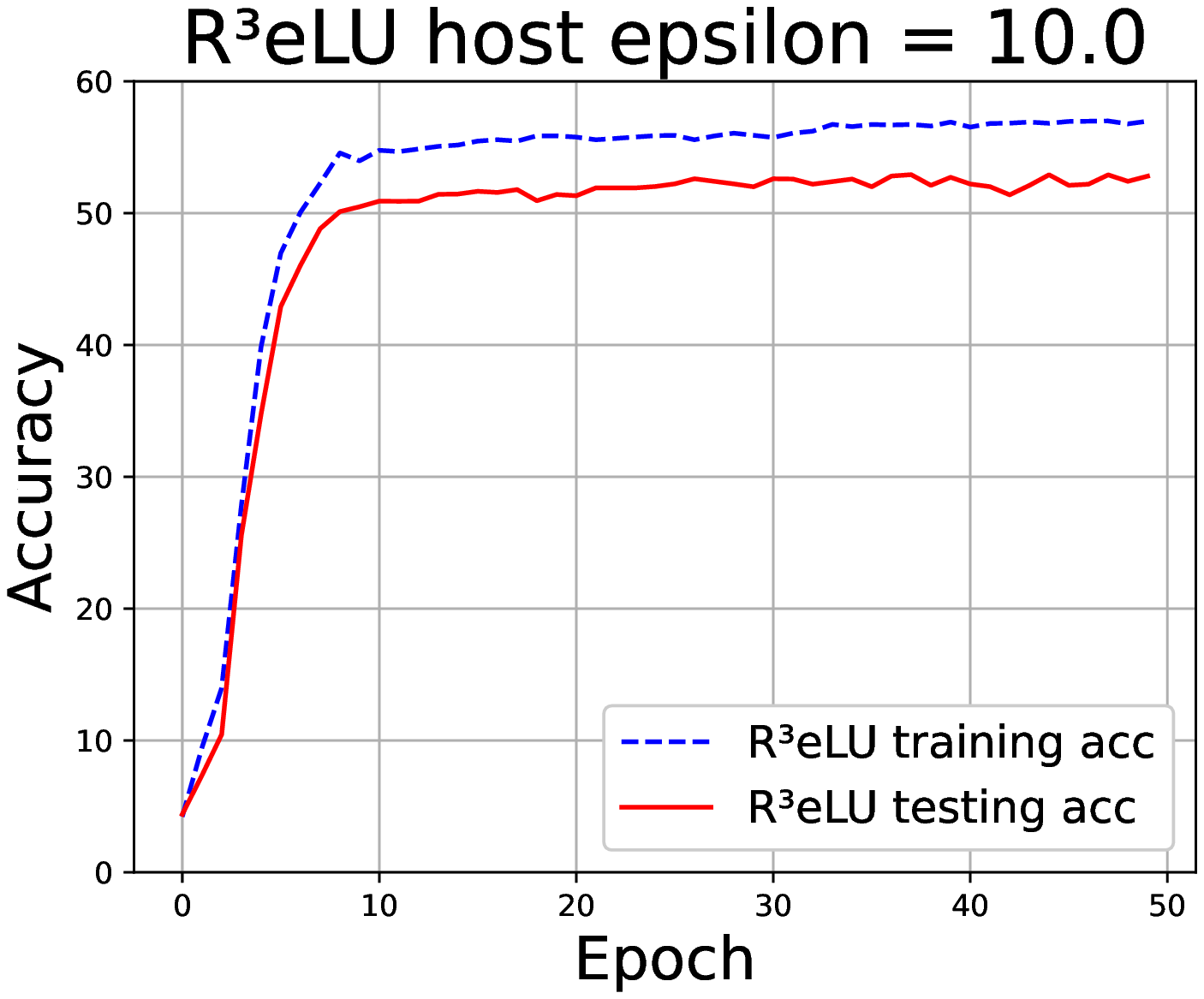}
    \end{minipage}
    \caption{SplitNN learning curve with the host's privacy protected by our solution.}
    \label{fig:Host accuray on Movielens}
\end{figure*}


\begin{table}[ht]
\caption{Defense against an adversarial FSHA host using various budgets.}
\label{tab:defense_fsha_host_budget}
\small
\begin{tabular}{c|cc|cc|cc}
\hline
\multirow{2}{*}{$\epsilon$} & \multicolumn{2}{c|}{MovieLens}                                                                                               & \multicolumn{2}{c|}{BookCrossing}
                                    & \multicolumn{2}{c}{MNIST}
                                        \\ \cline{2-7} 
                          & baseline   & \textbf{Ours}  & baseline& \multicolumn{1}{c|}{\textbf{Ours}} & baseline & \textbf{Ours} \\
\cline{2-7} &1200 Epochs &\multicolumn{1}{c|}{5000 Epochs} &1200 Epochs &\multicolumn{1}{c|}{5000 Epochs} &9000 Epochs 
&\multicolumn{1}{c}{9000 Epochs}
                          \\ \hline
$0.1$  & \multirow{9}{*}{0.2652 $\times 10^{-3}$} & \multicolumn{1}{c|}{\textbf{455.5676}} &  \multirow{9}{*}{0.2365 $\times 10^{-3}$} & \multicolumn{1}{c|}{\textbf{500.1487}} & 
\multirow{9}{*}{0.0206} & \multicolumn{1}{c}{\textbf{1.98257}}\\

$0.25$ &   &\multicolumn{1}{c|}{\textbf{73.7824}} &  &\multicolumn{1}{c|}{\textbf{83.7606}} &  &\multicolumn{1}{c}{\textbf{1.9788}} \\
$0.5$ &   &\multicolumn{1}{c|}{\textbf{21.1706}} &  &\multicolumn{1}{c|}{\textbf{30.2905}} &  &\multicolumn{1}{c}{\textbf{1.9703}} \\
$0.75$ &   &\multicolumn{1}{c|}{\textbf{9.5312}} &  &\multicolumn{1}{c|}{\textbf{8.5984}} &  &\multicolumn{1}{c}{\textbf{1.9534}} \\
$1.0$ &   &\multicolumn{1}{c|}{\textbf{4.4903}} &  &\multicolumn{1}{c|}{\textbf{6.1267}} &  &\multicolumn{1}{c}{\textbf{1.9442}} \\
$2.0$ &   &\multicolumn{1}{c|}{\textbf{1.1559}} &  &\multicolumn{1}{c|}{\textbf{1.7047}} &  &\multicolumn{1}{c}{\textbf{1.9283}} \\
$4.0$ &   &\multicolumn{1}{c|}{\textbf{0.2719}} &  &\multicolumn{1}{c|}{\textbf{0.2478}} &  &\multicolumn{1}{c}{\textbf{1.9209}} \\
$6.0$ &   &\multicolumn{1}{c|}{\textbf{0.1091}} &  &\multicolumn{1}{c|}{\textbf{0.1638}} &  &\multicolumn{1}{c}{\textbf{1.9135}} \\
$8.0$ &   &\multicolumn{1}{c|}{\textbf{0.0975}} &  &\multicolumn{1}{c|}{\textbf{0.0846}} &  &\multicolumn{1}{c}{\textbf{1.9116}} \\
\hline
\end{tabular}
\end{table}

In Table~\ref{tab:benchmark-cutlayer}, we give a benchmark of SplitNN using different cut layers with two public datasets, \cite{harper2015Movielens} and BookCrossing \cite{ziegler2005improving}. We also give the top-10 hit ratio for the test in Table~\ref{tab:benchmark-cutlayer}. We use min as our merging strategy. We combine one linear layer with one ReLU as one cut layer. We notice that there is little difference between different cut layers. However, considering the computational cost at the guest part, it is a good tradeoff between computational cost and model usability to select the first layer as the cut layer.
\begin{table}[ht]
\centering
\caption{Top-10 hit ratio (\%) of SplitNN using different cutlayers.}
\label{tab:benchmark-cutlayer}
\begin{tabular}{ccccccc}
\hline

\multicolumn{2}{c}{}    &layer1     &layer2     &layer3     &layer4     &no split
                                                                            \\ \hline

\multicolumn{1}{c|}{\multirow{2}{*}{MovieLens}}                     & padding     & \textbf{57.19} & 56.97 & 56.72 & 57.07 & \multirow{2}{*}{57.21}   \\ \cline{2-6}
\multicolumn{1}{c|}{}   & non-padding & 55.08 & 55.76 & \textbf{56.94} & 56.75 &     \\ \hline
\multicolumn{1}{c|}{\multirow{2}{*}{Book Crossing}} & padding  & 60.98 & 60.98 & 61.24 & \textbf{61.12} & \multirow{2}{*}{61.92}  \\ \cline{2-6}
\multicolumn{1}{c|}{}  & non-padding & 59.02 & 58.93 & \textbf{60.16} & 59.83 &      \\ \hline
\end{tabular}
\end{table}

\subsection{Model Architecture}
The neural networks we used for MovieLens, BookCrossing, MNIST, and CIFAR100 datasets after a split are shown in Table~\ref{tab:arch}. These networks are widely used in related studies. We split them according to the interpretation of SplitNN in previous studies \cite{ceballos2020splitnn,Pasquini2021UnleashingTT}.

\begin{table}[ht]
\centering
\caption{Model architectures used for evaluation.}
\label{tab:arch}
\begin{tabular}{cccccc}
\hline
\multicolumn{3}{c|}{MovieLens Model}                                            & \multicolumn{3}{c}{BookCrossing Model}                                  \\ \hline
Guest Layers            & Dim.               & \multicolumn{1}{c|}{Param.} & Guest Layers              & Dim.               & Param. \#            \\
Linear(160,128)+ReLU    & 128                  & \multicolumn{1}{c|}{20608}     & Linear(160,128)+ReLU      & 128                  & 20608                \\
\multicolumn{1}{l}{}    & \multicolumn{1}{l}{} & \multicolumn{1}{l|}{}          & \multicolumn{1}{l}{}      & \multicolumn{1}{l}{} & \multicolumn{1}{l}{} \\
Host Layers             & Dim.               & \multicolumn{1}{c|}{Param.} & Host Layers               & Dim.               & Param.            \\
Linear(160,128)+ReLU    & 128                  & \multicolumn{1}{c|}{20608}     & Linear(160,128)+ReLU      & 128                  & 20608                \\
Merge Guest Output      &                      & \multicolumn{1}{c|}{}          & Merge Guest Output        &                      &                      \\
Linear(128,128)+ReLU    & 128                  & \multicolumn{1}{c|}{16512}     & Linear(128,256)+ReLU      & 256                  & 33024                \\
Linear(128,64)+ReLU     & 64                   & \multicolumn{1}{c|}{8256}      & Linear(256,128)+ReLU      & 128                  & 32894                \\
Linear(64,3952)+Softmax & 3952                 & \multicolumn{1}{c|}{256880}    & Linear(128,17384)+Softmax & 17384                & 2242536              \\ \hline
\multicolumn{6}{c}{MNIST Model}                                                                                                                           \\ \hline
\multicolumn{4}{c}{Guest Layers}                                                                            & Dim.               & Param.          \\
\multicolumn{4}{c}{Linear(28*14,128)+BatchNormalization+ReLU}                                               & 128                  & 50304                \\
\multicolumn{4}{c}{Linear(128,64)}                                                                          & 64                   & 8256                 \\
\multicolumn{4}{l}{}                                                                                        & \multicolumn{1}{l}{} & \multicolumn{1}{l}{} \\
\multicolumn{4}{c}{Host Layers}                                                                             & Dim.               & Param.           \\
\multicolumn{4}{c}{Linear(28*14,128)+BatchNormalization+ReLU}                                               & 128                  & 50304                \\
\multicolumn{4}{c}{Linear(128,64)+BatchNormalization+ReLU}                                                  & 64                   & 8256                 \\
\multicolumn{4}{c}{Merge Guest Output}                                                                      &                      &                      \\
\multicolumn{4}{c}{Linear(64,64)+BatchNormalization+ReLU}                                                   & 64                   & 4160                 \\
\multicolumn{4}{c}{Linear(64,10)+Softmax}                                                                   & 10                   & 650                  \\ \hline

\multicolumn{6}{c}{Cifar100 Model}                                     \\ \hline
\multicolumn{4}{c}{Guest Layers}                                                                            & Dim.               & Param.          \\
\multicolumn{4}{c}{conv1\_x = Conv2D(3, 64, kernel=3, padding=1)}                              & [1, 64, 32, 32]            & 1728\\
\multicolumn{4}{c}{+ BatchNormalization+ReLU}    & [1, 64, 32, 32]            & 128\\
\multicolumn{4}{c}{conv2\_x = BasicBlock(64, 64, stride=1)}               & [1, 64, 32, 32]            & 73984\\
\multicolumn{4}{c}{+ BasicBlock(64, 64, stride=1)}               & [1, 64, 32, 32]            & 73984\\
\multicolumn{4}{l}{}                                                                                        & \multicolumn{1}{l}{} & \multicolumn{1}{l}{} \\
\multicolumn{4}{c}{Host Layers}                                                                             & Dim.               & Param.           \\
\multicolumn{4}{c}{conv1\_x = Conv2D(3, 64, kernel=3, padding=1)}                              & [1, 64, 32, 32]            & 1728\\
\multicolumn{4}{c}{+ BatchNormalization+ReLU}    & [1, 64, 32, 32]            & 128\\
\multicolumn{4}{c}{conv2\_x = BasicBlock(64, 64, stride=1)}               & [1, 64, 32, 32]            & 73984\\
\multicolumn{4}{c}{+ BasicBlock(64, 64, stride=1)}               & [1, 64, 32, 32]            & 73984\\
\multicolumn{4}{c}{Merge Guest Output}                                                                      &                      &                      \\
\multicolumn{4}{c}{conv3\_x = BasicBlock(64, 128, stride=2)}              & [1, 128, 16, 16]            & 230,144\\
\multicolumn{4}{c}{+ BasicBlock(128, 128, stride=2)}              & [1, 128, 16, 16]            & 295,424\\
\multicolumn{4}{c}{conv4\_x = BasicBlock(128, 256, stride=2)}              & [1, 256, 8, 8]            & 919,040\\
\multicolumn{4}{c}{ + BasicBlock(256, 256, stride=2)}              & [1, 256, 8, 8]            & 1,180,672\\
\multicolumn{4}{c}{conv5\_x = BasicBlock(256, 512, stride=2)}              & [1, 512, 4, 4]            & 3,673,088\\
\multicolumn{4}{c}{+ BasicBlock(512, 512, stride=2)}              & [1, 512, 4, 4]            & 4,720,640\\
\multicolumn{4}{c}{AdaptiveAvgPool2d(1, 1)+Linear(512, 1000)}                                                   & [1, 100]              & 51300\\ \hline
\end{tabular}
\end{table}

\end{document}